\documentclass{article}

\usepackage{arxiv}
\usepackage{hyperref}
\usepackage{url}

\usepackage{algorithm}
\usepackage{algorithmic}

\usepackage{amsmath,amsfonts,bm}

\def\eqref#1{equation~\ref{#1}}

\def\1{\bm{1}}

\DeclareMathAlphabet{\mathsfit}{\encodingdefault}{\sfdefault}{m}{sl}
\SetMathAlphabet{\mathsfit}{bold}{\encodingdefault}{\sfdefault}{bx}{n}

\usepackage[utf8]{inputenc} 
\usepackage[T1]{fontenc}    
\usepackage{url}            
\usepackage{booktabs}       
\usepackage{amsfonts}       
\usepackage{nicefrac}       
\usepackage{microtype}      
\usepackage{xcolor}         

\usepackage{subfigure}
\usepackage{subcaption}
\usepackage{standalone}
\usepackage{tikz}
\usepackage{pgfplots}
\pgfplotsset{compat=1.18}
\usepackage{multirow}
\usepackage{amsthm}
\usepackage{caption}
\usepackage{wrapfig}
\newtheorem{claim}{Claim}
\renewenvironment{proof}[1][Proof]{\noindent\textbf{#1.} }{\qed}
\usepackage{amsmath}
\usepackage{hyperref}       
\usepackage{url}            
\usepackage{lipsum}
\usepackage{graphicx}
\graphicspath{ {./images/} }
\usepackage{float}
\usepackage{natbib}
\usepackage{amssymb}
\usepackage{tabularx}
\setcitestyle{numbers}

\definecolor{hotpurple}{RGB}{180, 80, 200}

\title{Go Beyond Your Means: Unlearning with Per-Sample Gradient Orthogonalization}
\newcommand\ourmethod{OrthoGrad}

\author{
   Aviv Shamsian$^1$\thanks{Equal contribution, corresponding author: aviv.shamsian@live.biu.ac.il} \quad
  Eitan Shaar$^{2*}$ \quad
   Aviv Navon$^2$ \quad
  Gal Chechik$^{1,3}$ \quad
  Ethan Fetaya$^1$ \\ \\
$^{1}$Bar Ilan University  \quad
$^{2}$Independent Researcher \quad\
$^{3}$NVIDIA Research \quad
}

\begin{document}
\maketitle
\begin{abstract}
Machine unlearning aims to remove the influence of problematic training data after a model has been trained. The primary challenge in machine unlearning is ensuring that the process effectively removes specified data without compromising the model's overall performance on the remaining dataset.
Many existing machine unlearning methods address this challenge by carefully balancing gradient ascent on the `unlearn' data with the gradient descent on a `retain' set that represents the training data. However, in many cases the training dataset is not fully available when we wish to unlearn some concepts, because models are released without their training datasets, and one may only have access to a \textit{small part of a training set}. Here, we propose \ourmethod{}, a novel approach that mitigates interference between the unlearn set and a small retain set rather than competing ascent and descent processes. Our method projects the gradient of the unlearn set onto the subspace orthogonal to all gradients in the retain batch, effectively avoiding any gradient interference. We demonstrate the effectiveness of \ourmethod{} on multiple machine unlearning benchmarks, including automatic speech recognition, outperforming competing methods.
\end{abstract}

\section{Introduction}
Foundation models are trained on web-scale datasets, which may contain undesirable data: illegal, proprietary, or privacy-infringing. 
For example, Github Copilot~\citep{dakhel2023github, sirovs2024github} faced criticism for generating code snippets directly from open-source repositories without attribution, 
and the LAION-5B dataset~\citep{Schuhmann2022LAION5BAO}
had to be temporarily removed when it was discovered it contained CSAM images  \citep{thiel2023identifying}. 
Another type of undesirable data is the case were users may ask to `opt out' and to not be recognized by the system. For example, a user might want a speech recognition system to not transcribe his audio recordings. 
In all these cases, one is interested to ``remove'' or ``forget'' information from a pre-trained model, either general knowledge or specific information.

These challenges led to a growing recent interest in \emph{machine unlearning}~\citep{liu2024threats,nguyen2022survey}. In this setup, we wish to remove the effects of a given part of the training data on a pretrained model while preserving its generalization performance. In practice, we are given an \emph{unlearn set} that we wish to forget and a \emph{retain set} that represents the training data. Many existing methods~\citep{kurmanji2024towards,lin2024gdr} combine gradient ascent on an \textbf{unlearn set} -- for degrading performance on selected data,  with gradient descent on a \textbf{retain set} -- for preserving accuracy elsewhere. 

Very often however, models are released without their full training dataset, and one may only have access a small fraction of the training data to serve as a retain set. For instance, Whisper large-V3 ~\citep{radford2023robust}, an ASR foundation model, was trained on a proprietary dataset comprising over 5 million hours of labeled audio recordings. Although this private dataset cannot serve as a retain set, small-scale publicly available ASR datasets such as LibriSpeech can be used as substitutes. 
The key observation of this paper is that leading unlearning methods average over the retain set. However, when the retain set is small, one aims to go beyond averages and extract as much information as possible from the retain set.


In this work, we tackle the challenge of machine unlearning with a limited retain set. We propose a novel algorithm named \textit{\ourmethod{}}, which enables effective unlearning while minimizing the impact on the model's generalization performance. The key idea is to use the gradients over the retain set to estimate a subspace of gradients that should be maintained. This way, rather than relying heavily on the retain set to offset the negative effects of the unlearning process, our method directly mitigate interference by taking update steps that are orthogonal to the retain subspace.

To motivate our approach, we begin with a theoretical analysis under simplifying assumptions. The ideal objective of unlearning is to modify performance on the unlearn set while preserving performance on the retain set. This can be framed as an optimization problem constrained to the manifold of parameters that leave all retain-set points unaffected. We show that the gradient restricted to this manifold is equivalent to projecting the unlearning gradient onto the subspace orthogonal to the per-sample gradients of the retain batch. Inspired by this insight, we develop an algorithm that efficiently approximates the corresponding optimization trajectory. Unlike prior methods that rely on the average retain-set gradient, our approach adopts a \emph{per-sample gradient} perspective, yielding a more robust solution to unlearning (Figure~\ref{fig:grad_ortho}).


Our experiments focus on the challenging regime of small retain sets. These settings highlight the practical constraints often encountered in real-world applications of machine unlearning. We thoroughly evaluate the effectiveness of our approach, \ourmethod{}, across several challenging tasks, including image classification and automatic speech recognition. Additionally, we evaluate our approach across diverse unlearning regimes, including random data removal, class-specific forgetting, and a proxy-retain setting where the retain set is drawn from a related but distinct distribution, demonstrating versatility when the original training data are unavailable. Our results consistently show that \ourmethod{} achieves reliable unlearning while maintaining the overall model performance better than other leading unlearning methods.   

This paper makes the following contributions: (i) We propose \textit{\ourmethod{}} -- a new machine unlearning method, tailored for a limited amount of retain data. (ii) From a geometric perspective, we provide a theoretical motivation for our approach. (iii) We demonstrate the effectiveness of \ourmethod{} through extensive experiments spanning multiple datasets, modalities, and unlearning setups.



\begin{wrapfigure}{r}{0.45\textwidth}
    \vspace{-50pt}
    \begin{minipage}{\linewidth}
        \centering
        \includegraphics[width=\linewidth]{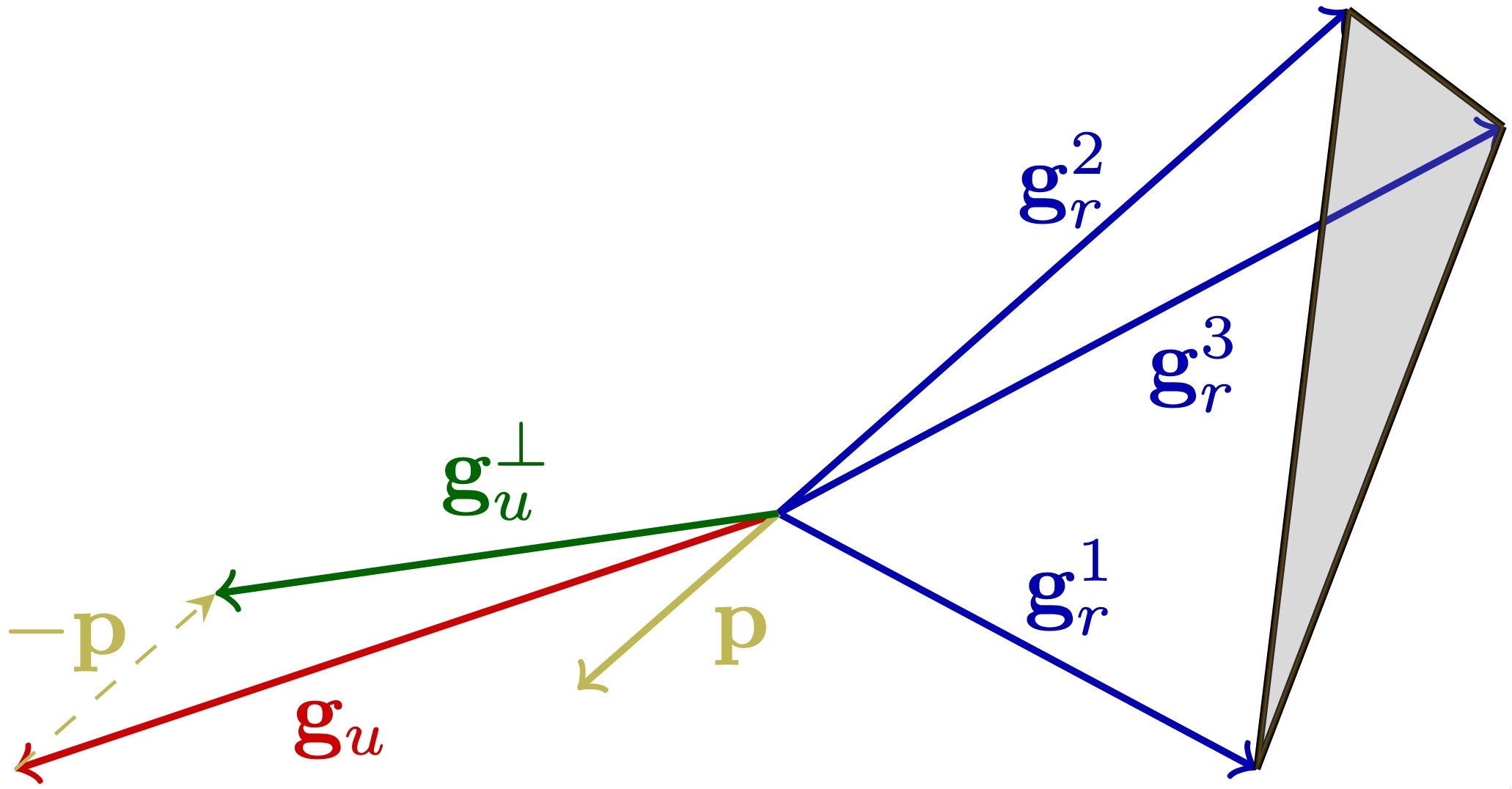}
        \caption{Illustration of the gradient orthogonalization process. The retain gradients \textcolor{blue!70!black}{$\mathbf{g}_r^1$}, \textcolor{blue!70!black}{$\mathbf{g}_r^2$}, and \textcolor{blue!70!black}{$\mathbf{g}_r^3$} span a subspace (gray triangle). The projection vector \textcolor{yellow!70!black}{$\mathbf{p}$} is obtained by applying QR decomposition on the retain gradients.
        The unlearn gradient \textcolor{red!70!black}{$\mathbf{g}_u$} is projected using \textcolor{yellow!70!black}{$\mathbf{p}$} to form unlearning gradient which is orthogonal to the retain subspace, \textcolor{green!40!black}{$\mathbf{g}_u^{\perp}$}.}
        \label{fig:grad_ortho}
    \end{minipage}
    \vspace{-10pt}
\end{wrapfigure}

\section{Related Work}
The development of efficient machine unlearning methods \citep{cao2015towards, fan2025challenging, ginart2019making, goel2022towards, zhang2024unlearncanvas, romero2007incremental, mehta2022deep, huang2025learning} has gained significant attention, addressing a range of applications across domains such as regression tasks \citep{thudi2022unrolling}, federated learning \citep{liu2021federaser, liu2022right, wang2022federated}, graph neural network \citep{chen2022graph, cheng2023gnndelete}. Retraining the model from scratch, widely regarded as the gold standard for unlearning \cite{fan2023salun}, guarantees the complete removal of data influence. However, this approach is often impractical in production environments due to the extensive computational resources, especially for large-scale datasets. Alternatively, fine-tuning a model for a new task may induce catastrophic forgetting \citep{lopez2017gradient}, but this mechanism fails to ensure the precise removal of specific data influences.

Most machine unlearning methods leverage techniques like influence functions \citep{guo2019certified, neel2021descent, wu2022puma, wu2025scissorhands, sekhari2021remember},  probabilistic approaches \citep{golatkar2020forgetting, golatkar2021mixed}. However, these methods often face inherent limitations that reduce their practical effectiveness, particularly in defending against membership inference attacks \citep{dwork2006our, graves2021amnesiac}.
As a result, the focus has shifted toward developing more effective and efficient unlearning strategies \citep{golatkar2020eternal, becker2022evaluating, jia2023model, chen2023boundary}. While these approaches represent significant advancements in machine unlearning, many rely on assumptions or techniques that limit their practicality in real-world scenarios. 

Cluster-based unlearning (DUCK~\citep{cotogni2023duck}, SCAR~\citep{bonato2024retain}), differing mainly in clustering metrics, lack adaptation to auto-regressive models, limiting applicability to sequential tasks. SCRUB~\citep{kurmanji2024towards}, a teacher–student framework, removes specific influences but struggles to generalize (e.g., forgetting random samples). GDR-GMA~\citep{lin2024gdr} relies on orthogonal projections of averaged gradients, ignoring per-sample variability and leaving residual influence. Most methods target classification; \citep{fan2023salun} highlights limitations for image generation, crucial for copyright and safety.


\paragraph{Conflicting Gradients in Multi-Task Learning:}
Multi-task learning (MTL) aims to improve model generalization by optimizing multiple related tasks~\citep{crawshaw2020multi, zhang2021survey}. However, different tasks often compete for model capacity and produce gradients pointing in opposite directions during training~\citep{yu2020gradient}. This phenomenon, known as gradient interference or conflict, occurs when a gradient that benefits one task degrades the performance of others. Addressing this issue by mitigating these conflicts has become crucial for training MTL systems, with early works focusing on analyzing conflict patterns and their relationships~\citep{sener2018multi, chen2018gradnorm}. Various optimization-based approaches have been proposed, including gradient projection and dropping to reduce task interference~\citep{chen2020just, wang2021gradient, liu2021conflict,achituve2024bayesian}. Recently, geometric and game-theory perspectives have led to methods seeking optimal Pareto solutions in the MTL optimization landscape~\citep{navon2022multi, javaloy2021rotograd}. Other studies proposed architectural solutions, including progressive networks~\citep{rusu2016progressive}, attention-based routing~\citep{ma2019snr}, and dynamic architecture adaptation~\citep{sun2020adashare}. Another line of works focuses on dynamic loss weighting, with methods like uncertainty weighting~\citep{kendall2018multi} and DWA~\citep{liu2019end} automatically balancing task losses based on pre-defined criteria. 
In the related field of continual learning, several projection-based methods constrain updates to avoid interfering with previously learned knowledge. For example, GEM/A-GEM enforce constraints using gradients on a small episodic memory of past data and project the current gradient accordingly \cite{Saha2021GradientPM,Chaudhry2018EfficientLL}, while orthogonality-based approaches explicitly encourage gradient directions that are orthogonal to protected subspaces \cite{Farajtabar2019OrthogonalGD,zeng2019continual}. These ideas closely connect to machine unlearning, which likewise requires an update that achieves forgetting while minimally degrading performance on retained data.

In this work, we focus on evaluating machine unlearning across various setups in the contexts of image classification and Automatic Speech Recognition (ASR). By exploring different unlearning scenarios, we aim to test the generalization capabilities of all methods, particularly in handling large-scale datasets and scenarios with restricted access to the original training data.

\section{Background}
We consider a training dataset $\mathcal{D} = \{(x_i, y_i)\}_{i=1}^N$ with each data point representing a pair of input vector $x_i$ and its corresponding label $y_i$. A machine learning model $f(\cdot;\theta)$, parameterized by parameters $\theta$, is optimized to minimize a loss function $\mathcal{L}(\theta)$. Formally, we define the loss function as $\mathcal{L}(\theta) = \frac{1}{N} \sum_{i=1}^N \ell(f(x_i; \theta), y_i)$, where $\ell$ is the cross-entropy loss. The model parameters trained on $\mathcal{D}$ are denoted as $\theta_p$ representing the \textit{pretrained model}. In machine unlearning, we are given two datasets: (i) Unlearn set $\mathcal{D}_u = \{(x_i, y_i)\}_{i=1}^{N_u}$, containing the $N_u$ data points to be unlearned. (ii) Retain set $\mathcal{D}_r = \{(x_i, y_i)\}_{i=1}^{N_r}$, with $N_r$ samples representing the training data to aid retain the model's performance. We assume these two datasets are disjoint, i.e., $\mathcal{D}_u \cap \mathcal{D}_r = \emptyset$. The primary goal of machine unlearning is to modify the model's weights to obtain $\theta_u$ resulting in an \textit{unlearned model} $f(\cdot;\theta_u)$. This modification process aims to remove the knowledge of the original model of $\mathcal{D}_u$. At the same time, the model must maintain its predictive performance on unseen data.

One major challenge in machine unlearning is how to define if the unlearn set was successfully unlearned. While theoretically, we want our model to be indistinguishable from a model trained from scratch without the unlearn set, this is hard to verify without actually retraining from scratch. In this work, similar to \cite{cotogni2023duck}, we aim for the performance on the unlearn set to match the original models' performance on the test set as a proxy. This also has the added benefit of making the comparison between different unlearning methods straightforward. As we care about both unlearn set performance and test set performance, by normalizing the unlearn set performance of all models (up to some tolerance), we can directly compare using a single metric, i.e., the test set performance. 

\section{Method}
We introduce \ourmethod{}, a novel machine unlearning approach designed to address unlearning with \textit{limited retain data}.  We observe that current unlearning methods 
perform
%
gradient ascent for unlearning and gradient descent for retention. Such approaches excessively depend on the retain set, because in a sense, we are simultaneously forgetting and retraining during the unlearning phase. {Such mixed objectives are known to be harder to stabilize and optimize. } Considering this, our approach aims to mitigate the negative effect of the unlearning step instead of fixing it using the retain set. {We propose to use a small retain set in a more efficient way, by computing a subspace of gradients that should not interfere with the retained data.}

\subsection{Geometric Motivation}
To motivate our method, we start by analyzing theoretically how we would perform ideal unlearning under strong simplifying assumptions. This will then guide the design of our practical algorithm.  
Intuitively, we are interested in the set of parameter vectors $\theta\in\mathbb{R}^d$ that maintains a constant loss over the retain set $\mathcal{D}_r$. Formally, let $\bar{\mathcal{L}}_r(\theta)=(\ell_1(\theta),...,\ell_{N_r}(\theta))$ define the vector of losses over the retain set, we are interested in performing unlearning in the level set $\bar{\Theta}:=\{\theta\in\mathbb{R}^d \mid \bar{\mathcal{L}}_r(\theta)=\bar{\mathcal{L}}_r(\theta_p)\}$, i.e., unlearning without changing the loss on elements of the retain set.

\begin{claim}
    Assuming that (i) The loss $\ell$ is continuously differentiable, and (ii) The Jacobian of the retain loss $\nabla\bar{\mathcal{L}}_r{\in}\mathbb{R}^{N_r\times d}$ is of full rank for all $\theta{\in}\bar{\Theta}$ then $\bar{\Theta}$ is a smooth manifold of dimension $d-N_r$
\end{claim}

\begin{proof}[Proof Sketch]
    To see that, we define for $\theta'\in\bar{\Theta}$ the function $f:\mathbb{R}^d\to\mathbb{R}^{N_r}$ by $f(\theta)=\bar{\mathcal{L}}_r(\theta)-\bar{\mathcal{L}}_r(\theta_p)$. From our assumptions, $f$ is continuously differentiable, and the Jacobian of $f$, $\nabla f= \nabla\bar{\mathcal{L}}_r$ has full rank at $\theta'$. Thus, as a direct result of the implicit function theorem, the set $\bar{\Theta}$ is locally diffeomorphic to an open ball in $\mathbb{R}^{d-N_r}$.
\end{proof}

Considering continuous parameter updates, to minimize the unlearn loss while remaining in $\bar{\Theta}$, we follow the gradient flow restricted to the manifold. This requires projecting the Euclidean gradient of our objective onto the tangent space $T_{\theta}\bar{\Theta}$, ensuring the flow stays on the manifold.

\begin{claim}
    The tangent space $T_{\theta'}\bar{\Theta}$ to $\bar{\Theta}$ at $\theta'$ is given by the null space of the Jacobian $\nabla_{\theta'}\bar{\mathcal{L}}_r$, that is, the set of directions in parameter space that are orthogonal to the subspace spanned by the retain gradients, $T_{\theta'}\bar{\Theta}=\{ v \in \mathbb{R}^d \mid \nabla\bar{\mathcal{L}}_r v =0 \}$
\end{claim}

\begin{proof}[Proof Sketch]
To show that $T_{\theta'}\bar{\Theta}=\text{Ker}(\nabla_{\theta'}\bar{\mathcal{L}}_r)$, we first note that $T_{\theta'}\bar{\Theta}\subset\text{Ker}(\nabla_{\theta'}\bar{\mathcal{L}}_r)$: Let $v\in T_{\theta'}\bar{\Theta}$ and let $\gamma(\cdot)$ be a smooth curve in $\bar{\Theta}$ with $\gamma(0)=\theta'$ and $\dot{\gamma}(0)=v$, we have $0=\frac{d}{dt}f(\gamma(t))\mid_{t=0}=\nabla_{\theta'}\bar{\mathcal{L}}_r v$, and so $v\in \text{Ker}(\nabla_{\theta'}\bar{\mathcal{L}}_r)$. Finally, we get the equality $T_{\theta'}\bar{\Theta}=\text{Ker}(\nabla_{\theta'}\bar{\mathcal{L}}_r)$ following a dimension counting argument since $\text{dim}(T_{\theta'}\bar{\Theta})=\text{dim}(\text{Ker}(\nabla_{\theta'}\bar{\mathcal{L}}_r))=d-N_r$.
\end{proof}

To algorithmically perform this gradient flow we would need to compute the standard gradient, project it to the space orthogonal to the gradients of the entire retain set, and then update the parameters along the exponential map, or update and then project back to the manifold. This, however, is very demanding computationally, as we need to compute and store the gradients on the entire retain set, as well as compute the exponential map or projection step.

\subsection{Practical Algorithm}
\label{sec:practical_alg}

\begin{wrapfigure}{r}{0.48\textwidth}
    \vspace{-20pt}
    \begin{minipage}{\linewidth}
        \begin{algorithm}[H]
            \small
            \caption{\textit{\ourmethod{}}}
            \label{alg:orthograd}
            \begin{algorithmic}
               \STATE {\bfseries Input:} Forget set $\mathcal{D}_u$, retain set $\mathcal{D}_r$, learning rate $\eta$, \\ combination parameter $\alpha$
               \STATE {\bfseries Output:} Updated model parameters $\theta_p$
               \STATE Apply LoRA modules to the pretrained model:\\
               $\theta_l = LoRA(\theta_p)$
               \REPEAT
               \STATE Sample a batch $\mathcal{B}_u \subset \mathcal{D}_u$ and $\mathcal{B}_r \subset \mathcal{D}_r$
               \STATE Compute the gradient $g_u$ from $\mathcal{B}_u$
               \STATE Compute the retain batch per-sample gradient matrix:\\ $G_r = [g_r^1, g_r^2, \dots, g_r^k]$ from $\mathcal{B}_r$
               \STATE Perform QR decomposition on $G_r$ to extract subspace:\\ 
               $Q = QR(G_r), \quad Q = [q_r^1, q_r^2, \dots, q_r^k]$
               \STATE Project $g_u$ onto the retain gradient subspace:\\
               $p_i = \langle g_u, q^i_r \rangle q^i_r$
               \STATE Compute the orthogonalized unlearn gradient:\\
               $g_u^\perp = g_u - \sum_{i=1}^k p_i$
               \STATE Compute the mean retain gradient:\\
               $\bar{g}_r = \frac{1}{k} \sum_{i=1}^k g_r^i$
               \STATE Combine gradients to form a unified update direction:\\
               $g = \alpha \bar{g}_r - (1-\alpha)g^\perp_u$
               \STATE Update model parameters:\\
               $\theta_l \leftarrow \theta_l - \eta g$ 
               \UNTIL{Convergence or maximum number of iterations}
               \STATE Merge LoRA modules:\\
               $\theta_p = Merge(\theta_p, \theta_l)$
            \end{algorithmic}
        \end{algorithm}
    \end{minipage}
\end{wrapfigure}

While performing the exact gradient flow on the retain set is too computationally expensive to run in practice, it inspires the design of our simple and practical unlearning algorithm \ourmethod{}. At each optimization step, we simply project the unlearn gradient to the space orthogonal to all individual gradients of the retain batch. Specifically, at each step, we sample a batch of examples from the unlearn set $\mathcal{D}_u$ and calculate the mean gradient vector on this batch. We denote this gradient vector as $g_u$. Next, we sample a batch from the retain set $\mathcal{D}_r$. For the retain batch with $k$ samples, we compute the per-sample gradient matrix $G_r = [g_r^1, g_r^2, \dots, g_r^k]$, where each column $g_r^i$ corresponds to the gradient vector for sample $i$ in the batch. 
Importantly, we note that this can be achieved efficiently using modern automatic differentiation libraries, such as PyTorch~\citep{paszke2019pytorch}, which allow us to obtain per-sample gradients in a single forward-backward pass. To ensure orthogonality between $g_u$ and the column space of $G_r$, we employ QR decomposition ~\citep{francis1961qr}  on $G_r$. This yields an orthonormal basis $Q = [q_r^1, q_r^2, \dots, q_r^k]$ that spans this subspace. Once the retain gradient subspace is defined, we project the unlearn gradient onto this subspace to compute its projection w.r.t each subspace vector. For a single retain gradient $g^i_r$, the projection is calculated as: $g_u^\perp = g_u - \sum_{i=1}^k \langle g_u, q^i_r \rangle q^i_r$. We note that while previous algorithms, for example \cite{lin2024gdr}, do try to mitigate interference between the unlearn and the retain set gradients, they achieve this on the batch-level average gradients and not on the gradients of the individual data points. We found in our experiments that the more strict per-element constraint, instead of working on the mean gradient gives a stronger performance (Section~\ref{sec:ablation}).


We now discuss two modifications of our method that we found to provide large empirical gains. First, instead of changing the entire weight space, we use low-rank adaptation (LoRA)~\citep{Hu2021LoRALA} to limit further the effect that unlearning has on the overall test performance. We note that in general parameter parameter-efficient fine-tuning (PEFT) is a rapidly evolving field, and how to utilize it for unlearning best 
have yet to be thoroughly explored.
Second, while our method is robust to retain-set size, it may underuse it; linearly combining retain and unlearn gradients improves the performance of the unlearned model compared to solely performing gradient ascent in the direction of the unlearn gradient. 
Therefore, we define the update gradient as: $g = \alpha \bar{g}_r - (1-\alpha)g^\perp_u$ where $\bar{g}_r = \frac{1}{k} \sum_{i=1}^k g_r^i$ is the retain gradient averaged over the batch, and $\alpha \in [0,1]$ is a hyperparameter that controls the trade-off between forgetting and retaining. Finally, we update the model parameters $\theta$ using the update rule: 
$\theta_l \leftarrow \theta_l - \eta g$. The step-by-step procedure is presented in Algorithm~\ref{alg:orthograd}.

In summary, \ourmethod{} enforces orthogonality between the unlearn and retain gradients, minimizing the interference between the updates of the unlearn set and retain set. \ourmethod{} is designed for low-data regimes (small retain sets), because unlike previous methods, it takes into account the subspace of gradients defined by the retain set, rather than average aggregates only. 
We demonstrate \ourmethod{} effectiveness on various datasets and model architectures in the next section.

\section{Experiments}\label{sec:exp}
We evaluate \ourmethod{} and compare it with recent machine unlearning approaches. We use several datasets, model architectures, and unlearning setups to demonstrate the effectiveness and versatility of \ourmethod{} in the regime of a limited number of retain data points. To encourage future research and reproducibility, we will make our code publicly available. Additional experimental results are presented in Appendix~\ref{app:exp_results}, including insightful analyses, ablation studies on key hyperparameters, and a detailed discussion of evaluation metrics.

\paragraph{Baselines.} 
We compare \ourmethod{} with recent machine unlearning baselines.  (1) Retrain - retraining from scratch without the unlearn set. \textit{We note that this baseline is inappropriate in the low data regime since it overfits the retain data, but we include it for completeness}. (2) Finetune - finetune the pretrained model solely with the retain set. (3) NegGrad~\citep{graves2021amnesiac, thudi2022unrolling} - a naive approach that performs gradient ascent steps on the unlearn set. (4) NegGrad+~\citep{kurmanji2024towards} - NegGrad with the additional goal of minimizing retain loss and preserving the model's knowledge on the retain dataset. (5) FISHER~\citep{golatkar2020eternal} - adds additive noise to the pretrained weights with a constraint on the fisher information matrix. (6) Influence~\citep{koh2017understanding, izzo2021approximate} - utilizes influence functions to identify the parameters most critical to the data being unlearned and perturb them by adding additive noise. (7) SCRUB~\citep{kurmanji2024towards} - A knowledge distillation approach that incorporates a regularization term into the unlearning objective. (8) DUCK~\citep{cotogni2023duck} - uses metric learning to minimize the distance between feature vectors of the data to be forgotten and the nearest centroid of a different class. (9) SCAR~\citep{bonato2024retain} - similar to DUCK, it uses Mahalanobis distance as the objective to minimize. (10) SSD~\citep{foster2024fast} - uses Fisher information to identify parameters tied to the forget set and selectively dampens them. We note that SCAR and SSD rely less on the retain set.
(11) GDR-GMA~\citep{lin2024gdr} - projecting conflicting gradients onto an orthonormal plane and dynamically adjusting the magnitude of update gradients.

\paragraph{Evaluation.}
We report the two common evaluation metrics in the field: (1) unlearning accuracy ($\mathcal{A}_{u}$) on the data to be forgotten, and (2) test accuracy ($\mathcal{A}_{test}$) on the held-out test set. For completeness, we also report retain accuracy ($\mathcal{A}_{r}$) on the retain data. This is comparable to train accuracy in standard learning and should not be used for comparison. In all experiments, we perform early stopping based on $\mathcal{A}_{u}$ reaching a specific target (normally the original test accuracy). This is easier to compare because the main difference is in $\mathcal{A}_{test}$.  Stopping criteria are crucial in machine unlearning to ensure the process reaches a proper balance between effective forgetting with retained functionality.

As machine unlearning involves multiple objectives, we propose the following \textit{Unlearning Impact Score} (UIS) for easier comparison.  Our metric is defined as: $$UIS= \left( \frac{|\mathcal{A}^p_{test} - \mathcal{A}^u_{test}|}{\mathcal{A}^p_{test}} + \frac{|\mathcal{A}^p_{test} - \mathcal{A}^u_{u}|}{\mathcal{A}^p_{test}} \right) / \,2\quad,$$ where the up scripts $p$ and $u$ denote pretrained and unlearned models respectively. 
In UIS we average two components: the relative change in test accuracy, and how close the performance on the unlearning set is to its target, $\mathcal{A}^p_{test}$. A lower UIS score indicates better unlearning, as it suggests the model has successfully forgotten the unlearn data while maintaining its performance on held-out data. Additional results with the MIA metric are in the appendix.


\begin{table}[b]
\centering
    \caption{\textit{Ablation Study.} Evaluation of \ourmethod{} variants on ASR unlearning. Values are word-error-rates averaged over 5 different speakers. }
    \label{tab:ablation}
    \resizebox{0.9\linewidth}{!}{
    \begin{tabular}{l cccc}
    \toprule
    & $\mathcal{W}_{retain}$ & $\mathcal{W}_{unlearn}$ & $\mathcal{W}_{speaker}$ & $\mathcal{W}_{test}$ \\
    \midrule
    \ourmethod{} Mean & $27.23 \pm 11.36$ & $96.67 \pm 6.02$ & $64.25 \pm 35.48$ & $29.42 \pm 14.07$ \\
    \ourmethod{} Per-sample & $18.71 \pm 4.04$ & $100.00 \pm 0.00$ & $96.40 \pm 7.04$ & $26.87 \pm 0.60$ \\
    \ourmethod{} Mean + Lora & $23.77 \pm 9.62$ & $92.12 \pm 7.34$ & $63.27 \pm 35.43$ & $41.21 \pm 25.67$ \\
    \ourmethod{} Per-sample + Lora & $12.73 \pm 1.43$ & $98.30 \pm 2.50$ & $81.16 \pm 23.97$ & $\underline{16.36} \pm \underline{0.32}$ \\
    \midrule
    \ourmethod{} & $12.11 \pm 0.65$ & $96.24 \pm 8.06$ & $98.53 \pm 3.28$ & $\textbf{13.98} \pm \textbf{0.58}$ \\
    \bottomrule
    \end{tabular}
    }
\end{table}

\subsection{Automatic Speech Recognition} \label{sec:asr}
Automatic Speech Recognition (ASR) is the process of converting spoken language into written text, a fundamental component in many real-world applications~\citep{malik2021automatic, alharbi2021automatic}. ASR foundation models like Whisper~\citep{radford2023robust} are trained on extensive datasets of transcribed web audio containing many hours of speech recordings. These models may inadvertently retain sensitive or proprietary information. Furthermore, individuals may request that an ASR system cannot accurately transcribe their voice as a way to preserve their privacy and identity.

\textbf{Speaker unlearning.} We focus on the task of forgetting audio data associated with a particular speaker, using 
Whisper-Tiny~\citep{radford2023robust} architecture and  LibriSpeech~\citep{panayotov2015librispeech} dataset, containing $1$K hours of English speech recordings. We establish the unlearning setup by selecting a single speaker from the training set to serve as the unlearn set. We randomly allocate $10\%$ from the unlearn set to evaluate our model on the unlearned speaker. Additionally, we randomly sample $10\%$ of the remaining training set to form the retain set. The test set is taken directly from the original LibriSpeech dataset.

\textbf{Eval metrics.} We evaluate performance using word error rate (WER), a standard metric that measures the percentage of words incorrectly transcribed by the model. We report WER for 4 sets: unlearn ($\mathcal{W}_{unlearn}$), retain ($\mathcal{W}_{retain}$), test ($\mathcal{W}_{test}$), and speaker held out ($\mathcal{W}_{speaker}$). The speaker held-out dataset comprises of unseen audio recordings of the unlearned speaker. Since Whisper tends to hallucinate~\citep{koenecke2024careless} by predicting unwanted words, we clip the WER at a maximum of $100\%$.

\subsubsection{Ablation Study}
\label{sec:ablation}
We begin with an ablation study to evaluate the relative contribution of each component in our approach. We run the unlearning process for $30$ epochs with an early stopping when $\mathcal{W}_{unlearn}$ reaches $75\%$. WER tends to jump significantly during the last epochs, which can lead to a final WER that is much higher than our stopping criteria. Although the exact threshold is somewhat arbitrary, we observed a rapid increase in WER beyond a certain point. 
We illustrate this behavior empirically in Appendix~\ref{app:wer-threshold}, which plots $\mathcal{W}_{unlearn}$ across epochs and shows a clear late-stage jump. Consequently, the metric typically crosses reasonable thresholds within a single step, making the stopping choice relatively insensitive.

In this experiment, we compare 5 variants. (i) \textit{OrthoGrad Mean}; Projecting the unlearn gradient to be orthogonal to the average retain gradient, (ii) \textit{OrthoGrad Per-sample}; Projecting the unlearn gradient to the space orthogonal to all individual sample gradients, (iii) \textit{OrthoGrad Mean/Per-sample+LoRA}; The latter methods when the update is restricted to low-rank adapters. (iv) \textit{OrthoGrad}; Our full method that combines gradient descent on the retain set. 

Table~\ref{tab:ablation} shows the results. Per-sample orthogonalization has two benefits. It reduces the mean Word-Error-Rate $\mathcal{W}_{test}$ and also reduces its variance by an order of magnitude. We observed that  OrthoGrad Mean is very unstable: it may work well with some speakers but performs poorly on others. As seen in Table~\ref{tab:ablation}, restricting per-sample unlearning of OrthoGrad to LoRA adapters improves $\mathcal{W}_{test}$ significantly. However, this is not the case for OrthoGrad Mean due to instability. We note that all methods passed the $75\%$ WER  threshold on the unlearn set with a large margin, but the OrthoGrad Mean performance on unseen audio from the speaker, $\mathcal{W}_{speaker}$, was below the target threshold. This means the unlearning did not generalize well to new recordings of the unlearned speaker. Finally, we see that adding the retain gradient can offer an additional improvement, but this improvement is somewhat limited. 

\begin{table*}[b]
\centering
\caption{\textit{Automatic Speech Recognition.} ASR speaker unlearning results on the LibriSpeech dataset. Values are word-error-rates averaged over $5$ different speakers.}
\label{tab:asr}
\resizebox{0.85\linewidth}{!}{
\begin{tabular}{lccccc}
\toprule

Method & $\mathcal{W}_{retain}$ & $\mathcal{W}_{unlearn}$ & $\mathcal{W}_{speaker}$ & $\mathcal{W}_{test}$ \\
\midrule
Original & $9.99 \pm 0.15$ & $11.12 \pm 4.91$ & $10.06 \pm 6.39$ & $11.08 \pm 0.00$ \\
Finetune & $0.06 \pm 0.01$ & $13.39 \pm 5.26$ & $12.54 \pm 7.48$ & $13.67 \pm 0.04$ \\
\midrule
NegGrad+ & $72.87 \pm 19.18$ & $77.08 \pm 35.99$ & $\underline{94.89} \pm \underline{6.78}$ & $85.90 \pm 10.72$ \\
SCRUB & $100.00 \pm 0.00$ & $100.00 \pm 0.00$ & $100.00 \pm 0.00$ & $100.00 \pm 0.00$ \\
GDR-GMA & $17.38 \pm 9.69$ & $93.28 \pm 7.58$ & $94.76 \pm 6.23$ & $\underline{32.52} \pm \underline{5.72}$ \\
\midrule
\ourmethod{} & $12.11 \pm 0.65$ & $96.24 \pm 8.06$ & $\textbf{98.53} \pm \textbf{3.28}$ & $\textbf{13.98} \pm \textbf{0.58}$ \\
\bottomrule
\end{tabular}
}
\end{table*}

\subsubsection{ASR Speaker Unlearning Results}
\label{sec:asr_results}
For speaker unlearning, we compare \ourmethod{} to SCRUB~\cite{kurmanji2024towards}, GDR-GMA~\cite{lin2024gdr}, and NegGrad+~\cite{kurmanji2024towards}. 
We exclude metric learning methods (DUCK and SCAR) as they are designed for classification and are unsuitable for ASR. Also, SSD relies on trained parameters, making it unsuitable for optimizing LoRA in this unlearning setup. See technical details and hyperparameter selection in Appendix~\ref{app:asr}.

The results are shown in Table~\ref{tab:asr}. All methods, except for the finetune baseline, successfully unlearned the target speaker. We hypothesize that the high $\mathcal{W}_{test}$ values for both NegGrad+ and SCRUB arise from the fact that they do not take into account the conflict between the unlearn and retain gradients. In contrast, \ourmethod{} and GDR-GMA, which consider this conflict, perform well on this benchmark. \ourmethod{} significantly outperforms GDR-GMA, on test WER. 

\begin{table*}[t]
\caption{\textit{Proxy-Retain with ResNet18 architecture.} Performance is measured under two unlearning scenarios: random sampling of training data (3-seed average) and class removal (3-class average).}
\label{tab:proxy-retain}
\setlength{\tabcolsep}{4pt}
\small
\resizebox{\textwidth}{!}{
\begin{tabular}{lccccccccc}
\toprule
& \multicolumn{4}{c}{Random Sampling} && \multicolumn{4}{c}{Class Forgetting} \\
\cmidrule{2-5} \cmidrule{7-10}
Method & $\mathcal{A}_u$ & $\mathcal{A}_r$ & $\mathcal{A}_{test}$ & UIS ($\downarrow$) && $\mathcal{A}_u$ & $\mathcal{A}_r$ & $\mathcal{A}_{test}$ & UIS ($\downarrow$) \\
\midrule
Original & $96.10 \pm 0.28$ & $49.45 \pm 1.09$ & $81.97 \pm 0.00$ & -- && $97.31 \pm 1.21$ & $48.55 \pm 1.71$ & $81.97 \pm 0.00$ & -- \\
Retrain & $29.95 \pm 3.11$ & $99.70 \pm 0.43$ & $30.48 \pm 2.65$ & -- && $0.00 \pm 0.00$ & $99.94 \pm 0.09$ & $28.18 \pm 1.94$ & -- \\
FT & $78.61 \pm 1.38$ & $72.30 \pm 0.56$ & $67.78 \pm 1.17$ & -- && $29.23 \pm 14.15$ & $99.95 \pm 0.05$ & $64.26 \pm 0.45$ & -- \\
\midrule
NegGrad & $43.63 \pm 32.62$ & $24.30 \pm 13.96$ & $37.24 \pm 26.17$ & $0.507 \pm 0.359$ && $0.27 \pm 0.46$ & $20.43 \pm 10.87$ & $29.41 \pm 21.75$ & $0.322 \pm 0.130$ \\
NegGrad+ & $21.43 \pm 2.64$ & $28.36 \pm 1.67$ & $19.89 \pm 1.95$ & $0.748 \pm 0.028$ && $0.76 \pm 0.45$ & $42.07 \pm 4.62$ & $51.15 \pm 13.24$ & $0.193 \pm 0.080$ \\
FISHER & $10.53 \pm 0.60$ & $10.28 \pm 0.34$ & $10.18 \pm 0.30$ & $0.874 \pm 0.005$ && $71.15 \pm 32.00$ & $39.01 \pm 2.03$ & $63.81 \pm 0.49$ & $0.545 \pm 0.196$ \\
Influence & $10.22 \pm 0.55$ & $10.09 \pm 0.10$ & $10.00 \pm 0.00$ & $0.877 \pm 0.003$ && $72.01 \pm 32.84$ & $41.43 \pm 5.55$ & $73.97 \pm 7.63$ & $0.488 \pm 0.154$ \\
SCRUB & $40.21 \pm 6.51$ & $42.56 \pm 5.88$ & $38.63 \pm 4.30$ & $0.519 \pm 0.066$ && $1.20 \pm 1.19$ & $64.81 \pm 1.02$ & $52.36 \pm 2.05$ & $0.188 \pm 0.019$ \\
DUCK & $53.22 \pm 5.82$ & $99.47 \pm 0.18$ & $46.43 \pm 4.01$ & $0.392 \pm 0.060$ && $0.00 \pm 0.00$ & $42.53 \pm 2.76$ & $22.94 \pm 6.53$ & $0.360 \pm 0.040$ \\
GDR-GMA & $79.93 \pm 0.70$ & $93.97 \pm 0.89$ & $67.02 \pm 0.45$ & $\underline{0.104} \pm \underline{0.007}$ && $0.00 \pm 0.00$ & $51.74 \pm 7.53$ & $56.69 \pm 6.11$ & $0.154 \pm 0.037$ \\
SSD & - & - & - & - && $96.59 \pm 2.14$ & $48.12 \pm 0.95$ & $81.10 \pm 1.57$ & $0.595 \pm 0.005$ \\
SCAR & $78.05 \pm 4.29$ & $42.79 \pm 2.29$ & $68.32 \pm 3.60$ & $0.107 \pm 0.048$ && $0.09 \pm 0.09$ & $48.74 \pm 0.80$ & $64.44 \pm 6.22$ & $\underline{0.107} \pm \underline{0.037}$ \\
\midrule
\ourmethod{} & $80.99 \pm 1.21$ & $61.69 \pm 0.71$ & $68.41 \pm 0.79$ & $\textbf{0.089} \pm \textbf{0.012}$ && $0.46 \pm 0.42$ & $59.58 \pm 1.36$ & $68.94 \pm 3.81$ & $\textbf{0.082} \pm \textbf{0.021}$ \\
\bottomrule
\end{tabular}
}
\end{table*}

\subsection{Unlearning with Proxy Data}
In practice, the original training data are usually unavailable, especially for foundation models trained on copyrighted or proprietary data. As a result, practitioners who wish to perform unlearning must curate a small proxy retain set that approximates the original data distribution. To simulate this scenario, we evaluate \ourmethod{} in a proxy-retain setting using CINIC-10~\citep{darlow2018cinic}, which merges CIFAR-10 with resized ImageNet images from the same classes. We first train a ResNet-18 on CIFAR-10. During unlearning, we construct the retain set exclusively from the ImageNet-derived portion of CINIC-10, uniformly sampling 10\% of this pool. We use CIFAR-10 examples as both the forget set and the test set. This protocol enforces a distribution shift between retain and forget/test while preventing any retain-set leakage from CIFAR-10. Results are reported in Table~\ref{tab:proxy-retain}. Additionally, we visualize the unlearning (\(A_u\)) and generalization (\(A_{\text{test}}\)) trade off in Figure~\ref{fig:test_unlearn_tradeoff}.

\ourmethod{} achieves the lowest UIS in both random-sampling and class-forgetting, lowering \(A_u\) while keeping \(A_{\text{test}}\) near the pretrained model despite the proxy (distribution-shifted) retain set. In contrast, baselines leave residual memorization (high \(A_u\)) or cause large drops in \(A_r\) or \(A_{\text{test}}\). Other methods effectively fail to unlearn; for example, SSD did not achieve any unlearning in the random-forgetting setup, even after hyperparameter tuning. These results suggest that orthogonalizing updates to the retain gradient subspace provides better unlearning with scarce proxy data.

\subsection{Image Classification}
\label{sec:image_classification}
Image classification tasks are commonly used benchmarks for evaluating machine unlearning algorithms. These benchmarks have two variations: class-wise forgetting and random data forgetting. Class-wise forgetting focuses on removing the influence of an entire image class, while random data forgetting targets the removal of randomly selected data points from the training set. In the standard experimental setup, the entire training set, except for the unlearn set, is used as the retain set. We, however, are interested in the scenario where we have access to a limited retain set, and therefore subsample a portion of the training set to be our retain set. 



Our evaluation is conducted on the ImageNet \citep{deng2009imagenet} image classification dataset on both random sampling and class forgetting benchmarks. In the random unlearning setting, the unlearn set consists of 5K images sampled uniformly from the training data. In the class unlearning setting, the unlearn set comprises all training images belonging to the unlearn class. In both setups, we draw 10K images for the retain set and evaluate on the original test set. We use ResNet-18 \citep{he2016deep} and ViT \citep{dosovitskiy2020image} as our base classifiers. The stopping criteria used in the random forgetting experiments follow \cite{cotogni2023duck}, i.e., we stop when the unlearn accuracy is within a defined threshold (0.5\%) or lower than the test accuracy of the pretrained model. For class-wise forgetting, we stopped when the model's accuracy on the unlearned classes dropped below 1\%, indicating that the class had been effectively forgotten. In both cases, if the unlearning algorithm does not reach the target within a specific number of epochs, we report the results of the last epoch. 

\begin{table*}[h]
\centering
\caption{\textit{ImageNet using ResNet18.} Results are shown for two scenarios: random sampling (averaged over 3 seeds) and class forgetting (averaged over 3 classes).}
\label{tab:imagenet_resnet}
\setlength{\tabcolsep}{4pt}  
\small
\resizebox{\textwidth}{!}{
\begin{tabular}{lccccccccc}
\toprule
& \multicolumn{4}{c}{Random Sampling} && \multicolumn{4}{c}{Class Forgetting} \\
\cmidrule{2-5} \cmidrule{7-10}
Method & $\mathcal{A}_u$ & $\mathcal{A}_r$ & $\mathcal{A}_{test}$ & UIS ($\downarrow$) && $\mathcal{A}_u$ & $\mathcal{A}_r$ & $\mathcal{A}_{test}$ & UIS ($\downarrow$) \\
\midrule
Original & $79.04 \pm 0.68$ & $79.2 \pm 0.49$ & $69.76 \pm 0.00$ & $-$ && $91.76 \pm 4.79$ & $79.15 \pm 0.00$ & $69.76 \pm 0.00$ & $-$ \\
Retrain & $5.72 \pm 0.16$ & $95.90 \pm 0.32$ & $5.72 \pm 0.35$ & $-$ && $0.00 \pm 0.00$ & $84.60 \pm 1.35$ & $5.34 \pm 0.14$ & $-$ \\
FT & $76.81 \pm 0.83$ & $94.59 \pm 0.08$ & $67.85 \pm 0.03$ & $-$ && $78.02 \pm 18.48$ & $94.28 \pm 0.03$ & $67.39 \pm 0.03$ & $-$ \\
\midrule
NegGrad & $69.94 \pm 0.14$ & $70.43 \pm 1.52$ & $62.5 \pm 0.70$ & $0.053 \pm 0.003$ && $22.53 \pm 16.39$ & $74.81 \pm 2.09$ & $66.11 \pm 2.12$ & $0.187 \pm 0.129$ \\
NegGrad+ & $78.98 \pm 0.56$ & $79.23 \pm 0.37$ & $69.74 \pm 0.00$ & $0.066 \pm 0.004$ && $28.92 \pm 21.62$ & $75.94 \pm 1.72$ & $67.1 \pm 1.78$ & $0.226 \pm 0.177$ \\
FISHER & $78.85 \pm 0.70$ & $78.96 \pm 0.44$ & $69.59 \pm 0.09$ & $0.066 \pm 0.005$ && $91.64 \pm 4.87$ & $79.03 \pm 0.00$ & $69.64 \pm 0.00$ & $0.657 \pm 0.034$ \\
Influence & $78.98 \pm 0.68$ & $79.22 \pm 0.45$ & $69.74 \pm 0.02$ & $0.066 \pm 0.004$ && $78.82 \pm 18.09$ & $78.90 \pm 0.17$ & $69.52 \pm 0.12$ & $0.566 \pm 0.128$ \\
SCRUB & $78.86 \pm 0.70$ & $84.03 \pm 0.48$ & $69.47 \pm 0.11$ & $0.067 \pm 0.006$ && $0.53 \pm 0.55$ & $83.61 \pm 0.01$ & $69.42 \pm 0.09$ & $\underline{0.006} \pm \underline{0.004}$ \\
DUCK & $67.22 \pm 1.29$ & $99.88 \pm 0.00$ & $61.87 \pm 0.01$ & $0.074 \pm 0.011$ && $0.00 \pm 0.00$ & $99.94 \pm 0.00$ & $61.37 \pm 0.74$ & $0.060 \pm 0.006$ \\
GDR-GMA & $67.35 \pm 1.32$ & $99.97 \pm 0.02$ & $66.98 \pm 0.38$ & $\underline{0.037} \pm \underline{0.011}$ && $0.74 \pm 0.29$ & $99.22 \pm 0.22$ & $64.74 \pm 0.38$ & $0.043 \pm 0.000$ \\
SSD & $76.28 \pm 1.61$ & $76.17 \pm 0.85$ & $67.14 \pm 0.90$ & $0.065 \pm 0.005$ && $0.00 \pm 0.00$ & $78.65 \pm 0.31$ & $69.14 \pm 0.23$ & $\textbf{0.004} \pm \textbf{0.001}$ \\
SCAR & $64.41 \pm 8.89$ & $75.51 \pm 3.52$ & $61.20 \pm 6.33$ & $0.100 \pm 0.109$ && $0.00 \pm 0.00$ & $76.47 \pm 0.40$ & $65.17 \pm 0.65$ & $0.033 \pm 0.005$ \\
\midrule
\ourmethod{} & $69.95 \pm 0.15$ & $82.27 \pm 0.87$ & $67.59 \pm 0.48$ & $\textbf{0.016} \pm \textbf{0.002}$ && $0.48 \pm 0.40$ & $77.24 \pm 1.22$ & $67.25 \pm 1.24$ & $0.021 \pm 0.005$ \\
\bottomrule
\end{tabular}
}
\end{table*}

The results are presented in Tables \ref{tab:imagenet_resnet}, \ref{tab:imagenet_vit}. These experiments show that our method consistently meets the unlearning target and achieves superior performance in most settings, particularly under random sampling. \ourmethod{} also generalizes effectively across different scenarios, performing competitively in class-forgetting tasks. In contrast, several baselines, such as SSD, SCRUB, and SCAR, lack consistency, performing well in one setup but poorly in another. Additionally, SCAR and DUCK are designed specifically for image classification (see Section~\ref{sec:asr}), and SCAR is only practical when a moderate amount of retain data is available (see Section~\ref{sec:robustness}). In conclusion, \ourmethod{} delivers the most robust performance in both unlearning setups while being task-agnostic.

\begin{table*}[h]
\caption{\textit{ImageNet with ViT architecture.} Performance is measured under two unlearning scenarios: random sampling of training data (3-seed average) and class removal (3-class average).}
\label{tab:imagenet_vit}
\setlength{\tabcolsep}{4pt}  
\small
\resizebox{\textwidth}{!}{
\begin{tabular}{lccccccccc}
\toprule
& \multicolumn{4}{c}{Random Sampling} && \multicolumn{4}{c}{Class Forgetting} \\
\cmidrule{2-5} \cmidrule{7-10}
Method & $\mathcal{A}_u$ & $\mathcal{A}_r$ & $\mathcal{A}_{test}$ & UIS ($\downarrow$) && $\mathcal{A}_u$ & $\mathcal{A}_r$ & $\mathcal{A}_{test}$ & UIS ($\downarrow$) \\
\midrule
Original & $94.26 \pm 0.19$ & $94.04 \pm 0.24$ & $81.06 \pm 0.00$ & $-$ && $98.1 \pm 1.62$ & $94.17 \pm 0.00$ & $81.06 \pm 0.00$ & $-$ \\
Retrain & $3.36 \pm 0.04$ & $98.72 \pm 0.06$ & $3.42 \pm 0.04$ & $-$ && $0.00 \pm 0.00$ & $95.08 \pm 0.15$ & $3.29 \pm 0.16$ & $-$ \\
FT & $89.96 \pm 0.56$ & $99.94 \pm 0.01$ & $75.03 \pm 0.30$ & $-$ & & $91.97 \pm 8.18$ & $99.92 \pm 0.02$ & $74.07 \pm 0.05$ & $-$ \\
\midrule
NegGrad & $24.2 \pm 13.23$ & $24.26 \pm 13.46$ & $21.06 \pm 14.06$ & $0.720 \pm 0.186$ && $0.41 \pm 0.34$ & $89.63 \pm 3.84$ & $76.53 \pm 4.16$ & $0.030 \pm 0.028$ \\
NegGrad+ & $78.75 \pm 3.09$ & $87.63 \pm 5.34$ & $72.16 \pm 4.53$ & $0.069 \pm 0.051$ && $0.05 \pm 0.07$ & $90.67 \pm 4.22$ & $77.17 \pm 4.73$ & $0.024 \pm 0.028$ \\
FISHER & $94.36 \pm 0.09$ & $94.01 \pm 0.19$ & $81.00 \pm 0.05$ & $0.082 \pm 0.000$ && $98.02 \pm 1.82$ & $94.14 \pm 0.00$ & $80.99 \pm 0.00$ & $0.605 \pm 0.011$ \\
Influence & $89.08 \pm 1.90$ & $91.61 \pm 1.78$ & $77.48 \pm 1.68$ & $0.071 \pm 0.001$ && $15.28 \pm 20.76$ & $91.54 \pm 2.21$ & $78.13 \pm 2.06$ & $0.112 \pm 0.121$ \\
SCRUB & $94.54 \pm 0.15$ & $96.00 \pm 0.11$ & $80.76 \pm 0.05$ & $0.084 \pm 0.001$ && $45.82 \pm 37.8$ & $96.18 \pm 2.91$ & $75.81 \pm 4.13$ & $0.314 \pm 0.286$ \\
DUCK & $76.75 \pm 0.47$ & $100.00 \pm 0.00$ & $69.45 \pm 0.49$ & $0.097 \pm 0.000$ && $0.00 \pm 0.00$ & $99.39 \pm 0.02$ & $70.68 \pm 0.29$ & $0.064 \pm 0.001$ \\
GDR-GMA & $80.41 \pm 0.34$ & $99.34 \pm 0.05$ & $75.47 \pm 0.16$ & $\underline{0.038} \pm \underline{0.003}$ && $0.15 \pm 0.16$ & $97.56 \pm 0.18$ & $77.03 \pm 0.28$ & $0.025 \pm 0.001$ \\
SSD & $93.5 \pm 0.19$  & $93.38 \pm 0.29$ & $80.35 \pm 0.17$ & $0.081 \pm 0.000$ && $0.00 \pm 0.00$ & $94.22 \pm 0.03$ & $80.9 \pm 0.03$  & $\textbf{0.001} \pm \textbf{0.000}$ \\
SCAR & $71.49 \pm 1.34$ & $93.53 \pm 0.17$ & $77.83 \pm 0.15$ & $0.079 \pm 0.009$ && $0.00 \pm 0.00$ & $93.26 \pm 0.14$ & $76.85 \pm 0.24$ & $0.026 \pm 0.001$ \\
\midrule
\ourmethod{} & $81.04 \pm 0.13$ & $97.59 \pm 0.11$ & $78.22 \pm 0.13$ & $\textbf{0.018} \pm \textbf{0.001}$ && $0.10 \pm 0.14$ & $94.26 \pm 0.04$ & $80.73 \pm 0.05$ & $\underline{0.002} \pm \underline{0.000}$ \\
\bottomrule
\end{tabular}
}
\end{table*}




\subsubsection{Robustness to retain size}
\label{sec:robustness}
Here, we assess the robustness of our method to variations in the retain set size. To do so, we revisit the random sampling image classification setup from Section~\ref{sec:image_classification}. Specifically, we experiment with varying retain dataset sizes, ranging from $1$K to $200$K samples, reporting the UIS values. The results are presented in Figure~\ref{fig:retain_ratio}. We note that NegGrad is not affected by the size of the retain set since it only performs gradient ascent in the direction of the unlearn set. Additionally, we exclude SCAR from this experiment, as it involves inverting a covariance matrix, which results in a non-invertible matrix in the extreme case of \(1\)K samples, and leads to memory overflow for \(150\)K and \(200\)K samples.
Our model consistently outperforms baseline methods across all retain set sizes.

\begin{figure}[h]
    \centering
    \begin{minipage}[b]{0.5\textwidth}  
        \centering
        \includegraphics[width=\linewidth]{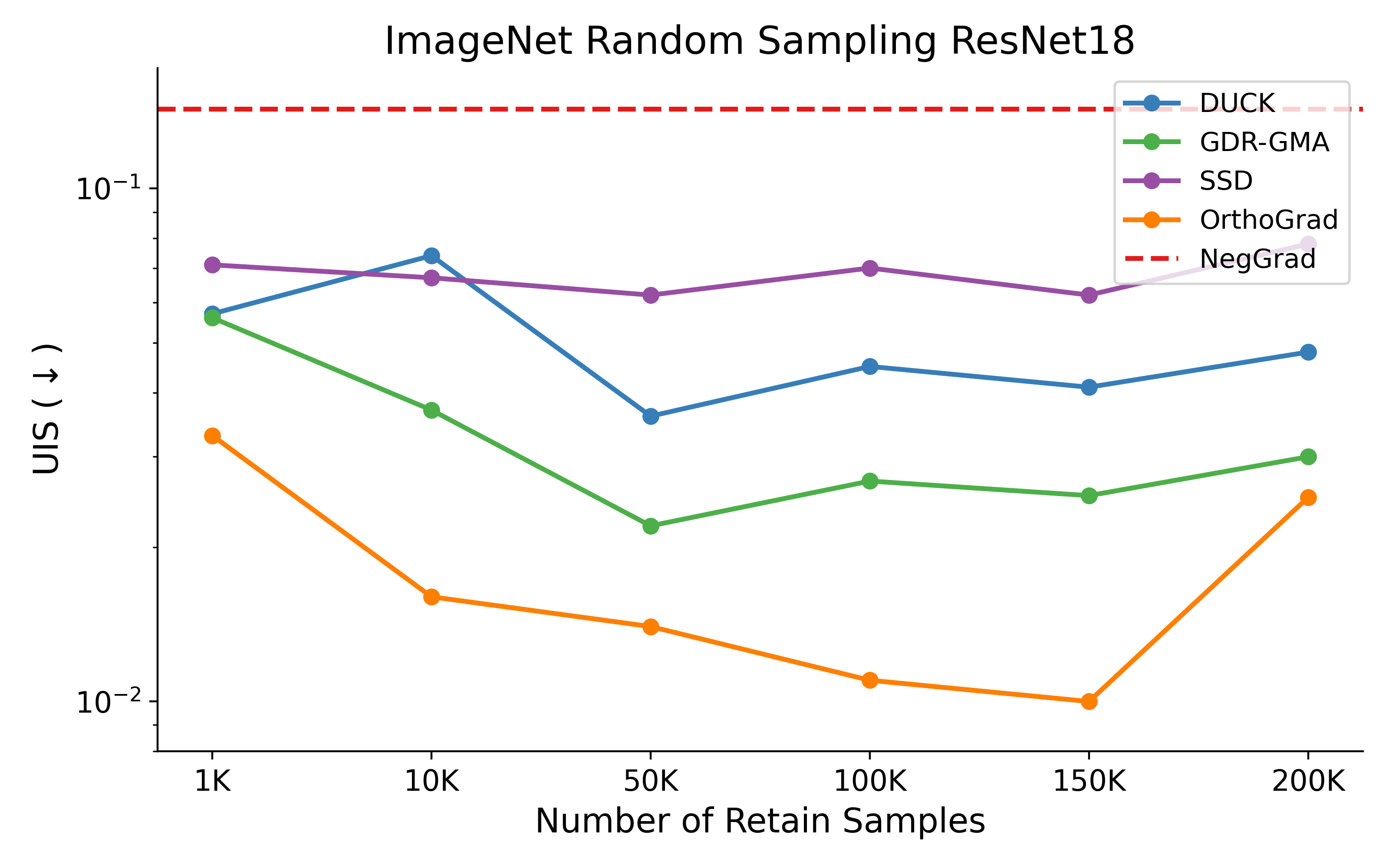}
    \end{minipage}%
    \hfill
    \begin{minipage}[b]{0.5\textwidth}
        \centering
        \includegraphics[width=\linewidth]{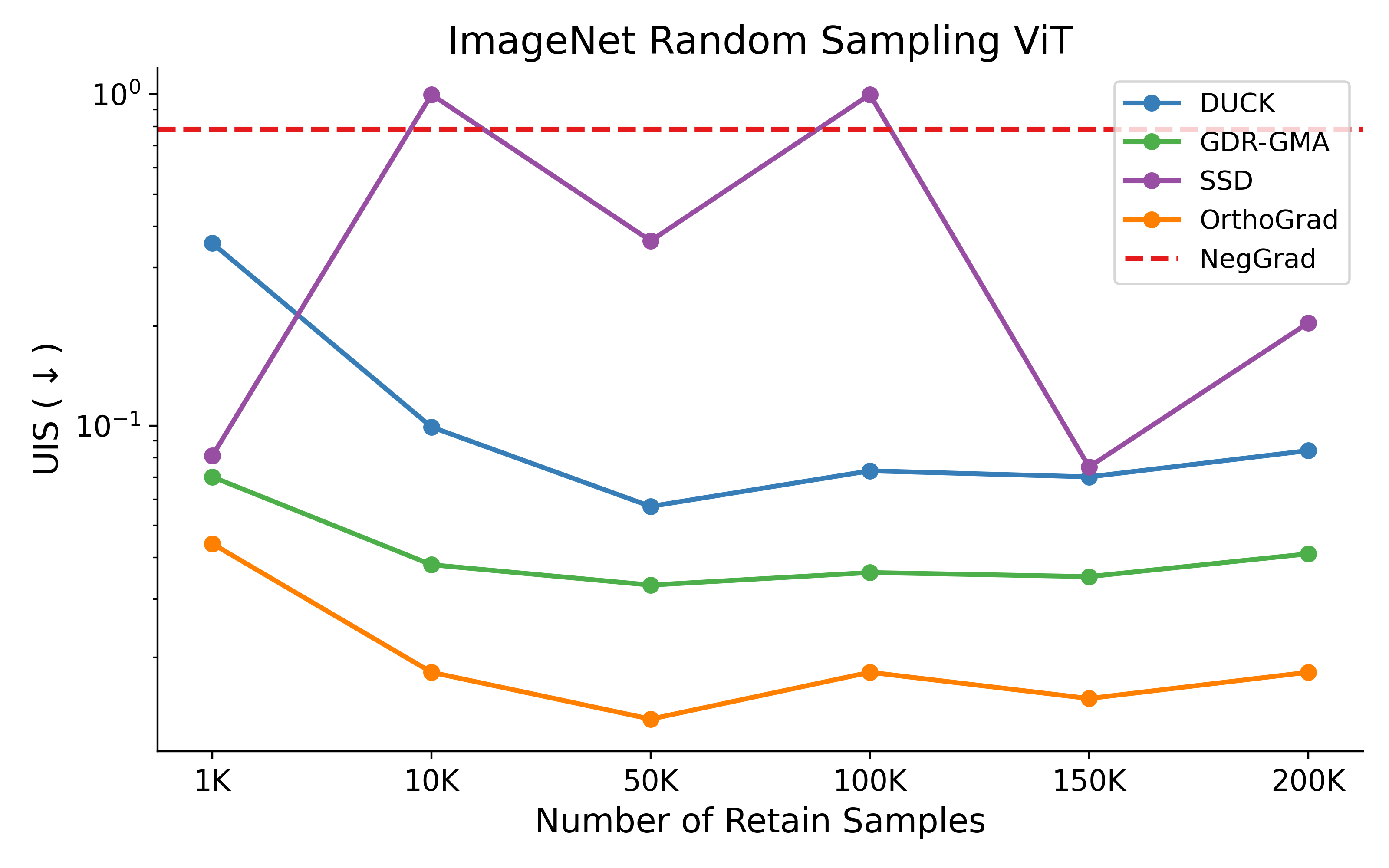}  
    \end{minipage}
    \caption{\textit{Varying retain samples.} We report UIS values across varying numbers of retained samples.}
    \label{fig:retain_ratio}
\end{figure}

\section{Runtime limitation}
In this work, we focus on the low-data regime, where \ourmethod{} yields consistent improvements over prior baselines. At the same time, \ourmethod{} relies on per-sample gradients for the retain batch, which can increase GPU memory consumption and introduce additional compute overhead compared to methods that only use averaged gradients. To that end, we revisit the ImageNet ViT experiment from Section~\ref{sec:image_classification} and report end-to-end wall-clock times for both class forgetting and random forgetting in Table~\ref{tab:vit_runtime}. Overall, \ourmethod{} remains in the same order of magnitude as other iterative gradient-based approaches (e.g., GDR-GMA), while being substantially faster than clustering-based methods (DUCK/SCAR) in this setup. Despite these considerations, \ourmethod{} presents a compelling solution for data-constrained unlearning settings.

\begin{table}[h]
\centering
\caption{\textit{ImageNet with ViT architecture.} Wall-clock runtime (seconds) for random sampling and class forgetting.}
\label{tab:vit_runtime}
\normalsize
\begin{tabular}{lcc}
\toprule
\textbf{Method} & \textbf{Random (sec)} & \textbf{Class (sec)} \\
\midrule
NegGrad    & 1,531  & 1,201  \\
NegGrad+   & 2,112  & 594    \\
FISHER     & 44,467 & 23,366 \\
Influence  & 778    & 743    \\
SCRUB      & 927    & 819    \\
DUCK       & 6,302  & 3,916  \\
GDR-GMA    & 1,105  & 654    \\
SSD        & 623    & 643    \\
SCAR       & 11,829 & 4,917  \\
\midrule
\ourmethod{}  & 2,657  & 1,239  \\
\bottomrule
\end{tabular}
\end{table}

\subsection{\ourmethod{} FLOPs Analysis}
To understand the computational profile of \ourmethod{}, we report the per-step compute cost in TFLOPs for the ImageNet with ViT-B16 architecture experiment (Table~\ref{tab:orthograd_flops}). Without LoRA, \ourmethod{} requires $29.84$ TFLOPs per step, while LoRA reduces this to $17.72$ TFLOPs, a $1.68\times$ reduction (40\% savings). The savings come from two sources: (1) the backward pass only computes gradients for LoRA adapters (884K vs.\ 86M parameters), reducing backward FLOPs by $\sim$2$\times$, and (2) the QR decomposition operates on smaller gradient matrices, dropping from $2.86$ to $0.03$ TFLOPs ($\sim$98$\times$ reduction) since its complexity is $\mathcal{O}(K \cdot B^2)$ where $K$ is the number of trainable parameters. These results demonstrate that LoRA provides substantial computational benefits beyond memory savings for \ourmethod{}.

\begin{table}[h]
\centering
\caption{\textit{ImageNet with ViT architecture.} FLOPs breakdown (TFLOPs) for a single \ourmethod{} update step with and without LoRA.}
\label{tab:orthograd_flops}
\small
\begin{tabular}{lccc}
\toprule
\textbf{Method} & \textbf{Grad} & \textbf{QR} & \textbf{Total} \\
\midrule
\ourmethod{} & 26.98 & 2.86 & 29.84 \\
\ourmethod{} + LoRA & 17.68 & 0.03 & 17.72 \\
\midrule
\textbf{Reduction} & 1.53$\times$ & 97.85$\times$ & 1.68$\times$ \\
\bottomrule
\end{tabular}
\end{table}

\section{Conclusions}
In this work, we focus on machine unlearning in the low data regime, where access to the retain data is limited. We present \ourmethod{}, a novel machine unlearning method that 
projects the aggregated unlearn gradient onto the subspace orthogonal to the individual gradients of the retain batch. We demonstrated the benefit of using per-sample gradient in the retain batch instead of averaging the retain gradients. Then, we demonstrate through various datasets, architectures, and unlearning setups the superiority of \ourmethod{} over existing machine unlearning methods. Since \ourmethod{} works well even without access to a large retain set, it can be applied in real-life use-cases where training data availability is constrained.





\newpage
\bibliographystyle{splncs04}
\bibliography{main}

\newpage
\appendix

\section{Additional Results}
\label{app:exp_results}

\subsection{Image classification}
We revisit Section~\ref{sec:image_classification} and evaluate \ourmethod{} on CIFAR-10 with a ResNet-18 backbone. In the random-sampling setting, we draw \(5{,}000\) images as the retain set and define the unlearn set as another \(5{,}000\) images sampled uniformly from the training data. In the class-forgetting setting, the unlearn set contains all training images from the designated unlearn class. In both cases, evaluation is performed on the standard test set. Results in Table~\ref{tab:cifar} mirror our earlier findings: \ourmethod{} reliably attains the unlearning objective, achieving superior performance in the class-forgetting setting while remaining comparable under random sampling.

\begin{table*}[h]
\caption{\textit{CIFAR10 with ResNet18.} Performance is evaluated across two unlearning scenarios: random sampling (3-seed average) and class forgetting (3-class average).}
\label{tab:cifar}
\setlength{\tabcolsep}{4pt}  
\small
\resizebox{\textwidth}{!}{
\begin{tabular}{lccccccccc}
\toprule
& \multicolumn{4}{c}{Random Sampling} && \multicolumn{4}{c}{Class Forgetting} \\
\cmidrule{2-5} \cmidrule{7-10}
Method & $\mathcal{A}_u$ & $\mathcal{A}_r$ & $\mathcal{A}_{test}$ & UIS ($\downarrow$) && $\mathcal{A}_u$ & $\mathcal{A}_r$ & $\mathcal{A}_{test}$ & UIS ($\downarrow$) \\
\midrule
Original & $96.1 \pm 0.28$ & $96.38 \pm 0.33$ & $81.97 \pm 0.00$ & $-$ && $97.3 \pm 1.21$ & $95.91 \pm 0.22$ & $81.97 \pm -$ & $-$ \\
Retrain & $60.43 \pm 0.22$ & $100.00 \pm 0.00$ & $61.04 \pm 0.35$ & $-$ && $0.00 \pm 0.00$ & $100.00 \pm 0.00$ & $55.55 \pm 1.03$ & $-$ \\
FT & $85.94 \pm 6.34$ & $88.49 \pm 7.44$ & $71.52 \pm 6.61$ & $-$ && $93.79 \pm 2.36$ & $100.00 \pm 0.00$ & $82.86 \pm 0.18$ & $-$ \\
\midrule
NegGrad & $40.39 \pm 22.59$ & $39.78 \pm 23.17$ & $34.94 \pm 22.87$ & $0.540 \pm 0.308$ && $0.00 \pm 0.00$ & $18.72 \pm 10.03$ & $15.43 \pm 9.40$ & $0.405 \pm 0.057$ \\
NegGrad+ & $81.42 \pm 0.63$ & $83.00 \pm 0.59$ & $69.25 \pm 1.08$ & $0.082 \pm 0.009$ && $24.54 \pm 33.85$ & $70.56 \pm 25.45$ & $56.42 \pm 24.31$ & $0.305 \pm 0.164$ \\
FISHER & $72.29 \pm 6.59$ & $72.57 \pm 6.55$ & $61.60 \pm 6.75$ & $0.183 \pm 0.089$ && $84.98 \pm 21.79$ & $69.08 \pm 5.77$ & $61.33 \pm 1.92$ & $0.645 \pm 0.143$ \\
Influence & $11.55 \pm 1.42$ & $11.59 \pm 1.30$ & $11.31 \pm 1.48$ & $0.860 \pm 0.019$ && $43.07 \pm 29.16$ & $66.72 \pm 31.99$ & $54.53 \pm 25.95$ & $0.431 \pm 0.043$ \\
SCRUB & $40.18 \pm 5.37$ & $42.51 \pm 4.88$ & $38.58 \pm 4.30$ & $0.519 \pm 0.066$ && $1.75 \pm 1.61$ & $87.32 \pm 3.09$ & $64.7 \pm 1.70$ & $0.116 \pm 0.022$ \\
DUCK & $86.46 \pm 0.19$ & $99.5 \pm 0.29$ & $78.02 \pm 0.35$ & $\underline{0.051} \pm \underline{0.002}$ && $0.00 \pm 0.00$ & $41.28 \pm 4.25$ & $35.35 \pm 4.82$ & $0.284 \pm 0.029$ \\
GDR-GMA & $81.75 \pm 0.40$ & $99.08 \pm 0.41$ & $71.6 \pm 0.28$ & $0.065 \pm 0.002$ && $0.19 \pm 0.19$ & $89.88 \pm 5.31$ & $66.79 \pm 3.30$ & $0.093 \pm 0.021$ \\
SSD & $96.10 \pm 0.23$ & $96.38 \pm 0.27$ & $81.97 \pm 0.00$ & $0.090 \pm 0.000$ && $0.040 \pm 0.06$ & $80.49 \pm 2.60$ & $61.99 \pm 1.41$ & $0.122 \pm 0.008$ \\
SCAR & $81.38 \pm 1.15$ & $96.07 \pm 0.17$ & $78.92 \pm 0.84$ & $\textbf{0.024} \pm \textbf{0.010}$ && $0.00 \pm 0.00$ & $91.39 \pm 3.09$ & $72.48 \pm 3.41$ & $\underline{0.058} \pm \underline{0.021}$ \\
\midrule
\ourmethod{} & $81.18 \pm 2.92$ & $93.27 \pm 0.71$ & $73.35 \pm 0.41$ & $0.058 \pm 0.005$ && $0.36 \pm 0.35$ & $97.57 \pm 0.43$ & $74.87 \pm 1.05$ & $\textbf{0.045} \pm \textbf{0.006}$ \\
\bottomrule
\end{tabular}
}
\end{table*}

\subsection{Ablation study - image classification}
In this section, we analyze the individual contributions of each component in \ourmethod{}, following a similar procedure to Section~\ref{sec:ablation}. Specifically, we revisit the image classification unlearning setup described in Section~\ref{sec:image_classification} and compare the following variants: OrthoGrad Mean, OrthoGrad Per-sample, OrthoGrad Mean/Per-sample+LoRA, and the full OrthoGrad approach. The experiments are conducted using a ResNet18 model trained on the CIFAR10 dataset for the unlearning task. The results, averaged over 3 seeds, are presented in Table~\ref{tab:img_ablation}.
Notably, \ourmethod{} achieves higher test accuracy and lower UIS values, indicating superior unlearning effectiveness without sacrificing generalization. This highlights the importance of combining both per-sample gradient components and the low-ranking optimization strategy.

\begin{table}[h]
\caption{\textit{Image classification ablation study.} Evaluation of \ourmethod{} variants on CIFAR10 random unlearning. Values are averaged over 3 different seeds. }
\label{tab:img_ablation}
\resizebox{\linewidth}{!}{
\begin{tabular}{l cccc}
\toprule
Method & $\mathcal{A}_u$ & $\mathcal{A}_r$ & $\mathcal{A}_{test}$ & UIS ($\downarrow$) \\
\midrule
\ourmethod{} mean & $82.47 \pm 0.51$ & $84.33 \pm 0.45$ & $70.06 \pm 0.19$ & $0.076 \pm 0.002$ \\
\ourmethod{} Per sample & $82.10 \pm 0.21$ & $83.21 \pm 0.45$ & $71.75 \pm 0.42$ & $0.064 \pm 0.002$ \\
\ourmethod{} Mean + LORA & $82.07 \pm 0.25$ & $84.83 \pm 1.03$ & $71.42 \pm 0.72$ & $0.066 \pm 0.004$ \\
\ourmethod{} Per sample + LORA & $82.01 \pm 0.26$ & $82.86 \pm 0.45$ & $72.04 \pm 0.27$ & $0.062 \pm 0.002$ \\
\midrule
\textbf{\ourmethod{}} & $\mathbf{81.35 \pm 1.2}$ & $\mathbf{87.44 \pm 1.27}$ & $\mathbf{73.34 \pm 1.32}$ & $\mathbf{0.058 \pm 0.015}$ \\
\bottomrule
\end{tabular}
}
\end{table}

\subsection{The relation between retain and unlearn gradients}
Here, we explore the relationship between the retain and unlearn gradients, which plays a central role in the effectiveness of our method. A natural concern is that if the gradients of the unlearn and retain sets are highly aligned, the orthogonal component of the unlearn gradient used in our projection step could be small. This may potentially weaken the unlearning effect. To address this concern, we analyze the cosine similarity between the retain and unlearn gradients over the course of training. Specifically, we repeat the class and random forgetting experiments on the ImageNet dataset using the ResNet18 architecture and report the cosine similarity between the unlearn and retain gradients. The results are presented in Figure~\ref{fig:cosine_plots}. The results show that, for the majority of the unlearning process, the gradients are not highly aligned—indicated by consistently non-zero cosine similarity values. These observations highlight the importance of the projection step in \ourmethod{}.

\begin{figure}[h]
    \centering
    \begin{minipage}{0.49\linewidth}
        \centering
        \includegraphics[width=\linewidth]{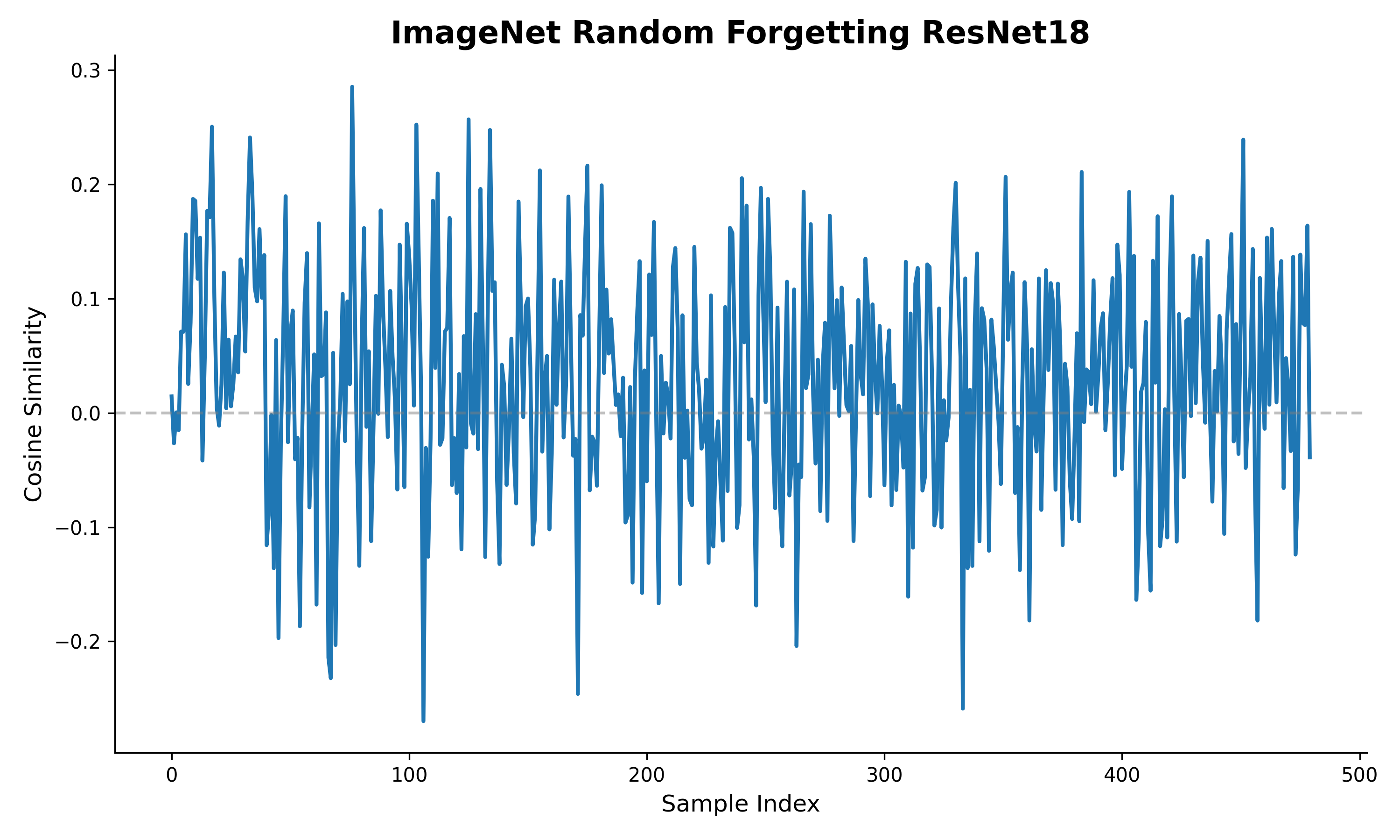}
    \end{minipage}
    \begin{minipage}{0.49\linewidth}
        \centering
        \includegraphics[width=\linewidth]{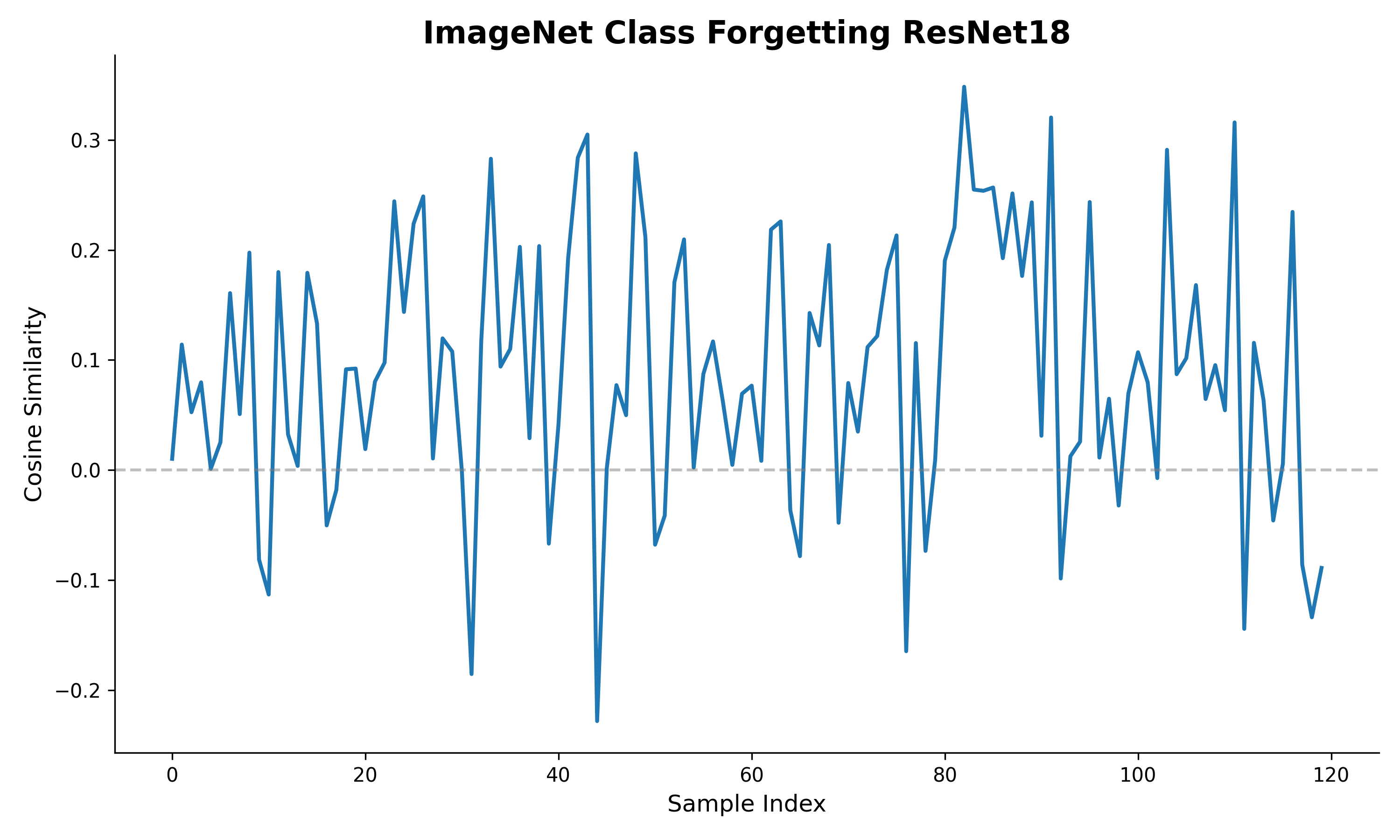}
    \end{minipage}
    \caption{Cosine similarity between retain and unlearn gradients during the unlearning process on ImageNet using ResNet18. The non-zero similarity values throughout training indicate that the gradients are not highly aligned, validating the importance of the projection step in \ourmethod{}.}
    \label{fig:cosine_plots}
\end{figure}

\subsection{Robustness to \(\alpha\)}
In Section~\ref{sec:practical_alg}, we describe the unlearning update direction defined by \ourmethod{}. This update rule incorporates a balancing parameter \(\alpha\), which interpolates between the unlearn and retain gradients. 
Here, we conduct an ablation study to assess the robustness of our method w.r.t \(\alpha\). Specifically, we run the class forgetting experiment on Imagenet using the ViT-b16 architecture across varying values \( \alpha \in \{0, 0.3, 0.7, 0.9\} \). The results are presented in Figure~\ref{fig:robust_retain}.
Notably, similar to the approach in~\cite{bonato2024retain}, \ourmethod{} is also capable of operating solely based on the unlearn gradient by setting \(\alpha {=} 0\). In this case, the update reduces to performing gradient ascent without relying on the retain set.

\begin{figure}[h]
    \centering
    \includegraphics[width=.8\columnwidth]{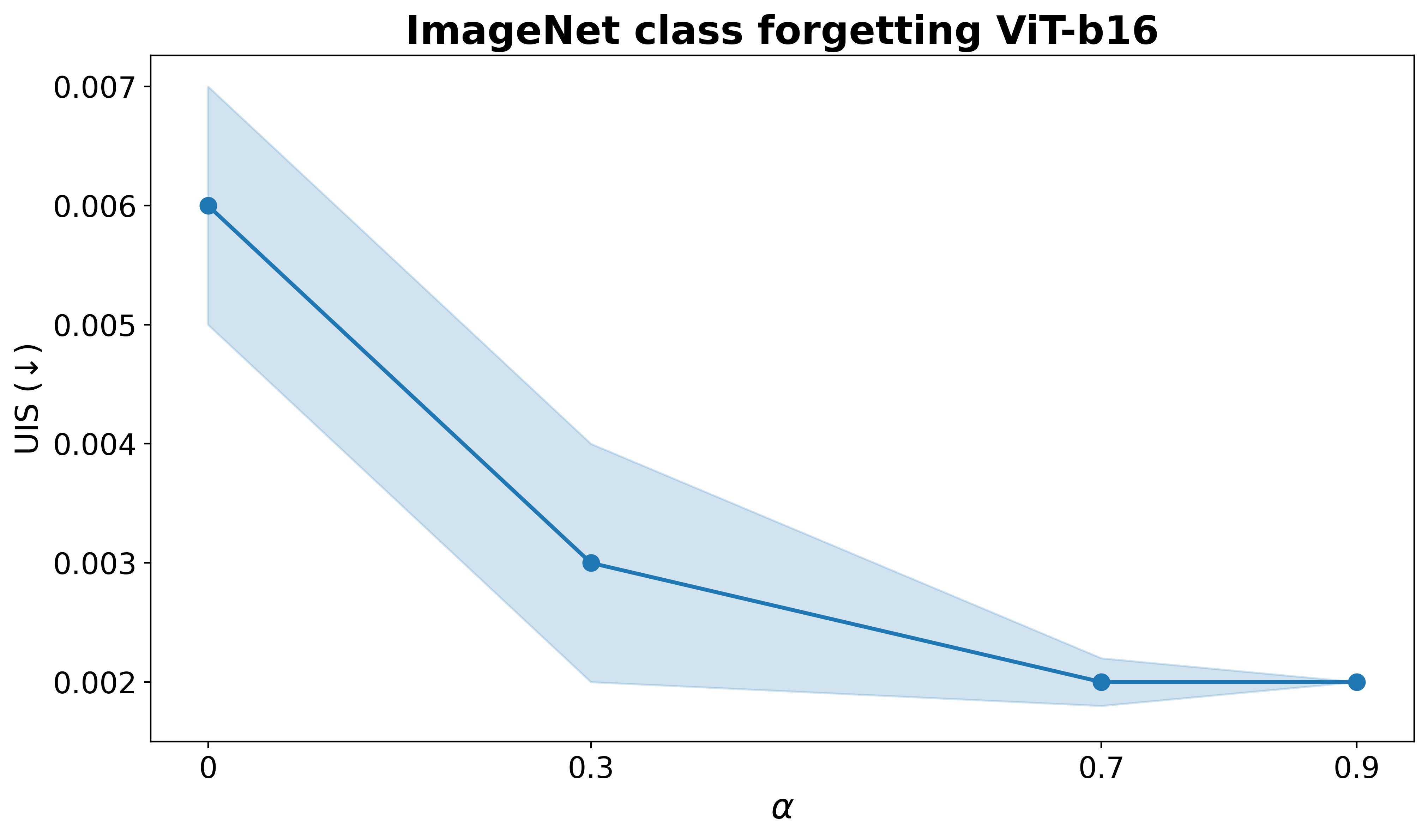}
    \caption{Comparison across different values of $\alpha$ for OrthoGrad on the ImageNet class forgetting setup using ViT.}
    \label{fig:robust_retain}
\end{figure}

\subsection{MIA as machine unlearning metric}
MIA can serve as a useful metric for evaluating machine unlearning. However, MIA has fundamental limitations that make it insufficient as a standalone measure of unlearning quality. Notably, an MIA score of 0, often interpreted as a perfect result, can indicate that the model has undergone catastrophic forgetting. In such cases, the model may not only forget the targeted data but also lose generalization capabilities. This shortcoming has been highlighted in prior work. \cite{foster2024fast} point out that low MIA accuracy may result from aggressive dampening that leads to degraded model performance. Similarly, \cite{hayes2024inexact} emphasized that relying solely on MIA can produce a misleading sense of unlearning efficacy, particularly when models are overly perturbed or when forgetting extends beyond the target data. Given these limitations, we believe MIA should be considered alongside other metrics that directly measure the trade-off between forgetting and retention. For completeness, we revisited the class forgetting benchmark on ImageNet using the ViT architecture and report the corresponding MIA scores. The results are presented in Table~\ref{tab:mia}. 

\begin{table}[h]
\caption{\textit{ImageNet using ViT-b16}. The results on class sampling setup are averaged over 3 classes.}
\label{tab:mia}
\begin{center}
\begin{tabular}{lccccc}
\toprule
Method & $\mathcal{A}_u$ & $\mathcal{A}_r$ & $\mathcal{A}_{test}$ & UIS ($\downarrow$) & MIA ($\downarrow$) \\
\midrule
Original & $94.26 \pm 0.19$ & $94.04 \pm 0.24$ & $81.06 \pm 0.00$ & $-$ & $-$ \\
Retrain & $0.00 \pm 0.00$ & $95.08 \pm 0.15$ & $3.29 \pm 0.16$ & $-$ & $-$ \\
FT & $91.97 \pm 8.18$ & $99.92 \pm 0.02$ & $74.07 \pm 0.05$ & $-$ & $-$ \\
\midrule
NegGrad & $0.41 \pm 0.34$ & $89.63 \pm 3.84$ & $76.53 \pm 4.16$ & $0.030 \pm 0.028$ & $0.23 \pm 0.39$ \\
NegGrad+ & $0.05 \pm 0.07$ & $90.67 \pm 4.22$ & $77.17 \pm 4.73$ & $\underline{0.024} \pm \underline{0.028}$ & $0.38 \pm 0.12$ \\
FISHER & $98.02 \pm 1.82$ & $94.14 \pm 0.00$ & $80.99 \pm 0.00$ & $0.605 \pm 0.011$ & $0.83 \pm 0.00$ \\
Influence & $15.28 \pm 20.76$ & $91.54 \pm 2.21$ & $78.13 \pm 2.06$ & $0.112 \pm 0.121$ & $0.04 \pm 0.06$ \\
SCRUB & $45.82 \pm 37.80$ & $96.18 \pm 2.91$ & $75.81 \pm 4.13$ & $0.314 \pm 0.286$ & $0.84 \pm 0.01$ \\
DUCK & $0.00 \pm 0.00$ & $99.39 \pm 0.02$ & $70.68 \pm 0.29$ & $0.064 \pm 0.001$ & $0.11 \pm 0.17$ \\
GDR-GMA & $0.15 \pm 0.16$ & $97.56 \pm 0.18$ & $77.03 \pm 0.28$ & $0.025 \pm 0.001$ & $0.54 \pm 0.41$ \\
\midrule
OrthoGrad & $0.10 \pm 0.14$ & $94.26 \pm 0.04$ & $80.73 \pm 0.05$ & $\textbf{0.002} \pm \textbf{0.000}$ & $0.30 \pm 0.24$ \\
\bottomrule
\end{tabular}
\end{center}
\end{table}

\subsection{Benefits of Integrating LoRA into OrthoGrad}

In Section~\ref{sec:practical_alg}, we detail the steps of \ourmethod{}, including the integration of LoRA modules. Here, we further elaborate on the motivation behind this design choice and highlight the specific benefits LoRA brings to scalable machine unlearning. One key advantage of using LoRA is its ability to localize parameter updates, which intuitively helps reduce unintended interference with retained knowledge. By limiting the impact of unlearning to a low-rank subspace, LoRA allows for more controlled and precise modifications, which aligns well with the goal of minimizing collateral forgetting. Additionally, LoRA is significantly more parameter-efficient than fine-tuning full model weights. This not only leads to reduced memory usage but also lowers the computational cost, making the approach viable for large-scale models like Whisper. These advantages collectively make LoRA a natural fit for \ourmethod{}, enabling both effective unlearning and practical deployment in real-world systems.

\subsection{Standard Machine Unlearning Setup}
This work addresses the low-data regime, where the available retain dataset is limited in size. We evaluate \ourmethod{} within the standard machine unlearning framework, in which the retain set coincides with the original training set used for the pretrained model. The experiments are conducted on the CIFAR10 dataset using the ResNet18 architecture, considering both random sampling and class forgetting scenarios. The results are presented in Table~\ref{tab:cifar_full}. We show that \ourmethod{} remains effective even when the retain set is relatively large, as further discussed in Section~\ref{sec:robustness}.

\begin{table*}[h]
\caption{\textit{CIFAR10 with ResNet18.} Performance is evaluated across two unlearning scenarios: random sampling (3-seed average) and class forgetting (3-class average).}
\label{tab:cifar_full}
\setlength{\tabcolsep}{4pt}  
\small
\resizebox{\textwidth}{!}{
\begin{tabular}{lccccccccc}
\toprule
& \multicolumn{4}{c}{Random Sampling} && \multicolumn{4}{c}{Class Forgetting} \\
\cmidrule{2-5} \cmidrule{7-10}
Method & $\mathcal{A}_u$ & $\mathcal{A}_r$ & $\mathcal{A}_{test}$ & UIS ($\downarrow$) && $\mathcal{A}_u$ & $\mathcal{A}_r$ & $\mathcal{A}_{test}$ & UIS ($\downarrow$) \\
\midrule
Original & $96.38 \pm 0.37$ & $96.09 \pm 0.04$ & $81.97 \pm 0.00$ & $-$ && $97.3 \pm 1.21$ & $95.98 \pm 0.13$ & $81.97 \pm 0.00$ & $-$ \\
Retrain & $82.09 \pm 0.3$ & $99.9 \pm 0.1$ & $82.6 \pm 0.63$ & $-$ && $0.00 \pm 0.00$ & $99.67 \pm 0.31$ & $74.22 \pm 1.56$ & $-$ \\
\midrule
NegGrad & $77.47 \pm 2.11$ & $77.84 \pm 1.76$ & $64.43 \pm 1.4$ & $0.134 \pm 0.024$ && $0.00 \pm 0.00$ & $19.23 \pm 11.41$ & $16.17 \pm 10.59$ & $0.401 \pm 0.064$ \\
NegGrad+ & $96.28 \pm 0.07$ & $96.05 \pm 0.34$ & $81.98 \pm 0.2$ & $0.088 \pm 0.002$ && $18.69 \pm 25.85$ & $68.64 \pm 25.74$ & $54.26 \pm 24.24$ & $0.283 \pm 0.120$ \\
FISHER & $96.3 \pm 0.48$ & $96.00 \pm 0.04$ & $81.83 \pm 0.1$ & $0.088 \pm 0.002$ &&  $0.06 \pm 0.05$ & $11.54 \pm 0.36$ & $10.37 \pm 0.265$ & $0.437 \pm 0.001$ \\
Influence & $96.31 \pm 0.4$ & $96.07 \pm 0.04$ & $81.96 \pm 0.05$ & $0.087 \pm 0.002$ && $17.54 \pm 12.92$ & $63.74 \pm 23.27$ & $50.92 \pm 17.18$ & $0.296 \pm 0.084$ \\
SCRUB & $55.58 \pm 2.26$ & $56.73 \pm 2.37$ & $55.23 \pm 3.14$ & $0.324 \pm 0.036$ && $0.02 \pm 0.00$ & $94.02 \pm 2.3$ & $72.91 \pm 0.76$ & $\underline{0.055} \pm \underline{0.004}$ \\
DUCK & $81.71 \pm 1.43$ & $89.34 \pm 2.06$ & $82.03 \pm 1.40$ & $\textbf{0.015} \pm \textbf{0.009}$ && $0.00 \pm 0.00$ & $79.12 \pm 13.5$ & $66.93 \pm 9.75$ & $0.091 \pm 0.072$ \\
GDR-GMA & $81.1 \pm 0.70$ & $86.18 \pm 3.26$ & $73.16 \pm 2.59$ & $0.059 \pm 0.020$ && $0.00 \pm 0.00$ & $87.99 \pm 4.33$ & $67.94 \pm 5.25$ & $0.085 \pm 0.032$ \\
SSD & $25.35 \pm 39.73$ & $94.95 \pm 1.22$ & $74.23 \pm 1.77$ & $0.200 \pm 0.230$ && $25.35 \pm 32.44$ & $94.95 \pm 1.00$ & $74.23 \pm 1.44$ & $0.202 \pm 0.189$ \\
SCAR & $80.11 \pm 1.65$ & $91.06 \pm 1.10$ & $79.21 \pm 1.28$ & $\underline{0.029} \pm \underline{0.016}$ && $0.00 \pm 0.00$ & $87.71 \pm 0.85$ & $70.75 \pm 1.32$ & $0.068 \pm 0.008$ \\
\midrule
\ourmethod{} & $81.35 \pm 1.2$ & $87.44 \pm 1.27$ & $73.34 \pm 1.32$ & $0.058 \pm 0.015$ && $0.67 \pm 0.29$ & $96.9 \pm 0.32$ & $75.67 \pm 0.93$ & $\textbf{0.042} \pm \textbf{0.007}$ \\
\bottomrule
\end{tabular}
}
\end{table*}

\subsection{Effectiveness of \ourmethod{} in the presence of larger forget sets}

To evaluate \ourmethod{} on larger forget sets, we extend the class forgetting setup to simultaneously remove three classes. Using the ResNet-18 architecture on CIFAR-10, we repeat the experiment across three different class combinations and report the mean and standard deviation of the results (Table~\ref{tab:img_unlearning}). The retain set consists of 5K images, consistent with the original setup.

\begin{table}[h]
\caption{\textit{Image classification unlearning results.} Comparison of \ourmethod{} and GDR-GMA on CIFAR10 class unlearning. Values are averaged over 3 different combinations of classes.}
\label{tab:img_unlearning}
\resizebox{\linewidth}{!}{
\begin{tabular}{l ccccccc}
\toprule
Method & $\mathcal{A}_u$ & $\mathcal{A}_r$ & $\mathcal{A}_{test}$ & $\mathcal{A}_{test}$ (w/o unlearned classes) & UIS ($\downarrow$) & UIS (w/o unlearned classes) ($\downarrow$) \\
\midrule
GDR-GMA & $0.48 \pm 0.04$ & $95.80 \pm 2.58$ & $56.94 \pm 0.41$ & $81.20 \pm 0.57$ & $0.155 \pm 0.002$ & $0.024 \pm 0.014$ \\
\textbf{\ourmethod{}} & $\mathbf{0.0 \pm 0.0}$ & $\mathbf{97.15 \pm 1.13}$ & $\mathbf{59.11 \pm 2.82}$ & $\mathbf{84.43 \pm 4.21}$ & $\mathbf{0.139 \pm 0.014}$ & $\mathbf{0.014 \pm 0.004}$ \\
\bottomrule
\end{tabular}
}
\end{table}

The overall accuracy is naturally lower because of the removal of three classes. Therefore, we also report the accuracy of the original test set after removing the samples of the unlearned classes. Nevertheless, \ourmethod{} achieves better unlearning performance and higher test accuracy compared to GDR-GMA.

Additionally, we conduct random forgetting experiments on ImageNet with the ResNet-18 architecture. The unlearn sets contain 50K, 150K, and 200K samples, corresponding to 10×, 30×, and 40× the size of the unlearn set in our main experiments. The retain set is fixed at 10K samples. The results are shown in Table~\ref{tab:imagenet_unlearn_sizes}.

\begin{table}[h]
\caption{\textit{Random unlearning on ImageNet with varying unlearn set sizes.} Evaluation of OrthoGrad vs. GDR-GMA on ImageNet with 50K, 150K, and 200K unlearn sets. Values shown are averaged over 3 seeds.}
\label{tab:imagenet_unlearn_sizes}
\resizebox{\linewidth}{!}{
\begin{tabular}{lcccccccccccc}
\toprule
\multirow{2}{*}{Method} 
& \multicolumn{4}{c}{50K samples} 
& \multicolumn{4}{c}{150K samples} 
& \multicolumn{4}{c}{200K samples} \\
\cmidrule(lr){2-5} \cmidrule(lr){6-9} \cmidrule(lr){10-13}
& $\mathcal{A}_u$ & $\mathcal{A}_r$ & $\mathcal{A}_{test}$ & UIS
& $\mathcal{A}_u$ & $\mathcal{A}_r$ & $\mathcal{A}_{test}$ & UIS
& $\mathcal{A}_u$ & $\mathcal{A}_r$ & $\mathcal{A}_{test}$ & UIS \\
\midrule
Original 
& 79.26 & 79.17 & 69.76 & -
& 79.32 & 79.17 & 69.76 & -
& 79.31 & 79.17 & 69.76 & - \\
GDR-GMA 
& 68.43 & 100.0 & 63.33 & 0.05566
& 68.85 & 100.0 & 62.13 & 0.06123
& 68.60 & 100.0 & 61.75 & 0.06573 \\
\ourmethod{} 
& 70.09 & 89.41 & 63.53 & \textbf{0.04700}
& 69.48 & 91.45 & 62.21 & \textbf{0.05608}
& 69.75 & 91.11 & 62.35 & \textbf{0.05317} \\
\bottomrule
\end{tabular}
}
\end{table}

These results show that \ourmethod{} remains effective with larger unlearn sets in the random forgetting setup. Compared to GDR-GMA, it achieves unlearn accuracy closer to the original model’s test accuracy, leading to better overall unlearning process.

\subsection{Batch Composition}
Here, we conduct an ablation study to examine the effect of retain set batch composition. Specifically, we compare the random batch sampling used in our main experiments with an alternative strategy in which each retain batch is restricted to a single class. To this end, we revisit the CIFAR10 class forgetting experiment and evaluate \ourmethod{} under both settings: (i) the original random sampling and (ii) single-class retain batches (excluding the forget class). Results are reported in Table~\ref{tab:ablation_same_class_batch}.

\begin{table}[h]
\caption{\textit{Ablation on batch composition.} Performance of \ourmethod{} with and without sampling from same-class batches.}
\label{tab:ablation_same_class_batch}
\centering
\resizebox{\linewidth}{!}{
\begin{tabular}{lcccc}
\toprule
Method & $\mathcal{A}_u$ & $\mathcal{A}_r$ & $\mathcal{A}_{\text{test}}$ & UIS ($\downarrow$) \\
\midrule
\ourmethod{} same class batch & $0.28 \pm 0.20$ & $94.28 \pm 1.26$ & $72.91 \pm 0.04$ & $0.057 \pm 0.001$ \\
\ourmethod{} & $0.36 \pm 0.35$ & $97.57 \pm 0.43$ & $74.87 \pm 1.05$ & $0.045 \pm 0.006$ \\
\bottomrule
\end{tabular}
}
\end{table}

The results suggest that using shuffled retain batches leads to better performance for \ourmethod{}. We hypothesize that this is because shuffled batches enable smoother optimization. In contrast, projecting the unlearn gradient onto the subspace spanned by a single class may constrain the optimization dynamics.

\subsection{Projection-Induced Signal Loss}
A potential concern with our approach is that by projecting the unlearn gradient away from the retain gradient directions, part of the unlearning signal could be lost. To study this, we measure the ratio between the projection component and the unlearn gradient,$\frac{\lVert g_u^{\perp} \rVert}{\lVert g_u \rVert}$,
where $g_u$ is the gradient from the forget set and $g_u^{\perp}$ is its projection component. This ratio indicates how much of the unlearning signal is preserved after projection.

We conduct experiments on CIFAR-10 with ResNet-18 under two setups: class forgetting and random forgetting. In each case, we track the projection ratio reporting its mean, standard deviation, and range. The results indicate that while part of the signal is removed during projection, the ratio remains consistently above zero, showing that sufficient unlearning signal is retained to remain effective. Specifically, in the class forgetting setting, we obtain $0.241 \pm 0.036$ (min: $0.181$, max: $0.597$), and in the random forgetting setting $0.229 \pm 0.028$ (min: $0.176$, max: $0.535$). The evolution of this ratio over training steps is illustrated in Figure~\ref{fig:norm_plots}.

\subsection{Sensitivity to the WER Stopping Threshold}
\label{app:wer-threshold}
In our ASR unlearning experiments, we use an early-stopping rule based on the unlearn-set WER, stopping once $\mathcal{W}_{\text{unlearn}}$ reaches a target threshold of $75\%$ and selecting the corresponding checkpoint. To empirically justify that the precise threshold value is not critical within a reasonable range, we visualize the evolution of $\mathcal{W}_{\text{unlearn}}$ over training. As shown in Fig.~\ref{fig:wer_jump_placeholder}, $\mathcal{W}_{\text{unlearn}}$ typically remains relatively stable for the first epochs and then exhibits a sharp jump in the final stages of training. This jump-like behavior implies that crossing a threshold (e.g., between $60\%$--$80\%$) often occurs within a single optimization epoch, so nearby threshold choices select checkpoints with very similar training dynamics. Overall, this provides empirical support that our ASR unlearning results are relatively insensitive to the exact WER stopping threshold, as long as it is chosen within a reasonable range.

\begin{figure}[h]
    \centering
    \begin{minipage}{0.49\linewidth}
        \centering
        \includegraphics[width=\linewidth]{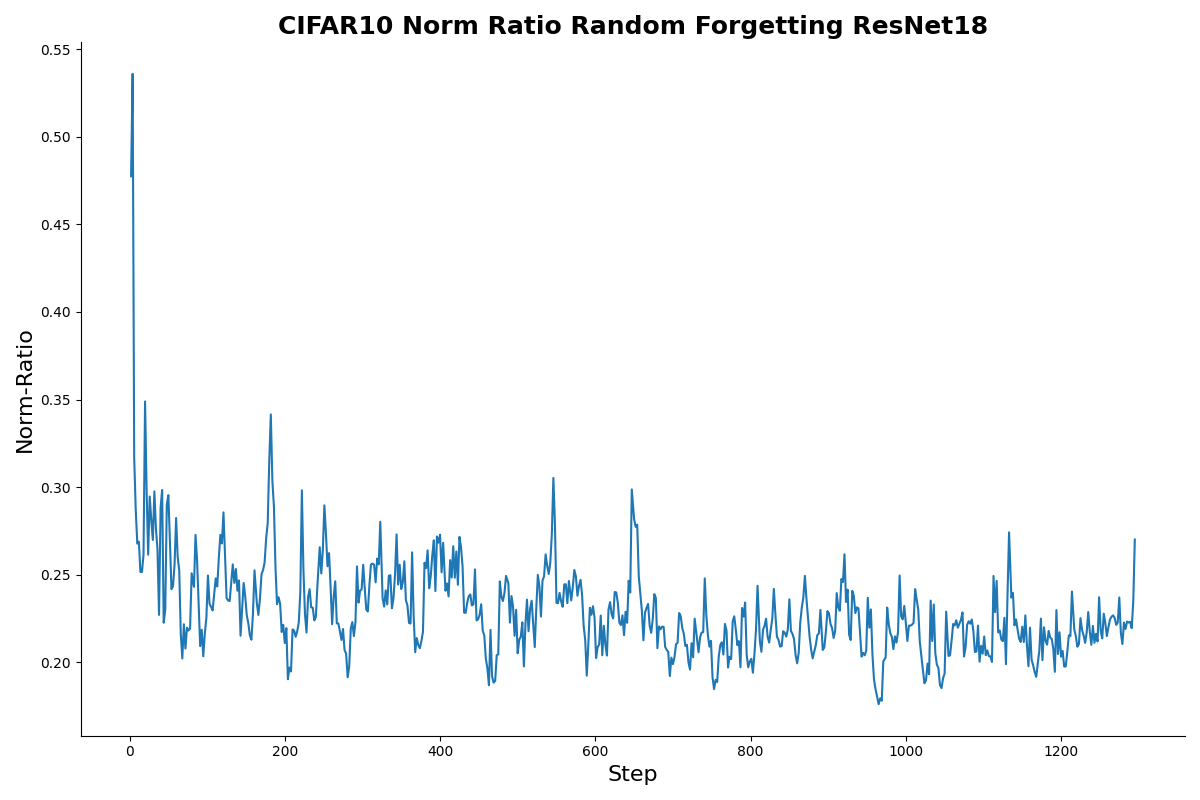}
    \end{minipage}
    \begin{minipage}{0.49\linewidth}
        \centering
        \includegraphics[width=\linewidth]{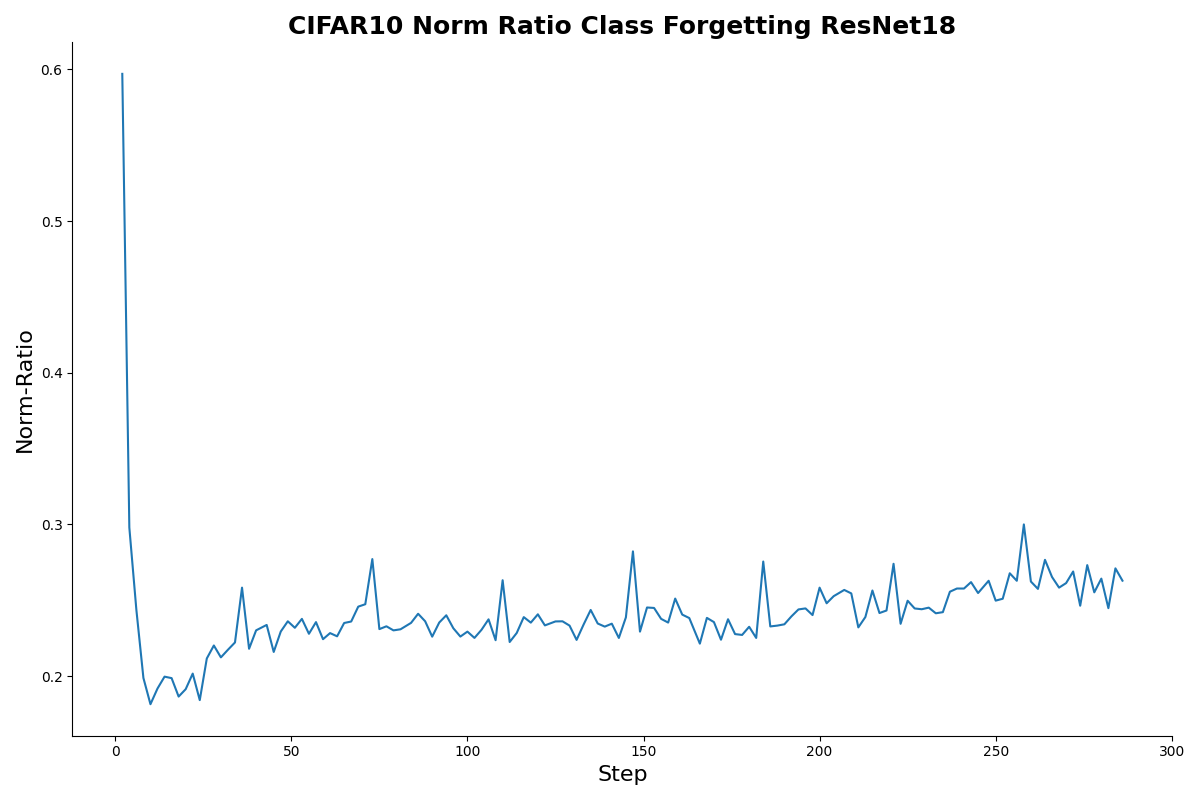}
    \end{minipage}
    \caption{Measured ratio between the projection component and the unlearn gradient conducted on CIFAR-10 with ResNet-18 under two setups class forgetting and random forgetting.}
    \label{fig:norm_plots}
\end{figure}

\section{Experimental details}
\label{app:experimental_details}
\subsection{Image classification}
\label{app:classification}
We provide additional details about the image classification machine unlearning setup, including general information and the hyperparameter search performed for each method. This setup leverages the splits from the CIFAR10 and ImageNet datasets, focusing on either random unlearning samples or specific class samples. Results are reported across three random seeds or three different classes. All methods are trained for 30 epochs, with each setup utilizing its corresponding stopping criteria as explained in \ref{sec:image_classification}.

\paragraph{Retrain.} We performed grid search for learning rate ($\eta$), and batch size for $\mathcal{D}_r$. Specifically, for $\eta$, we searched over $\{1e-3, 1e-2, 1e-1, 3e-2\}$, and for retain batch size, we searched over $\{256, 128\}$.
\textbf{For ViT architecture} The optimal parameters for random and class forgetting respectively $\eta = (3e-2, 3e-2)$, retain batch size = $(128, 128)$. \textbf{For ResNet18 architecture on ImageNet} The optimal parameters for random and class forgetting respectively $\eta = (1e-1, 1e-1)$, retain batch size = $(128, 128)$. \textbf{For ResNet18 architecture on CIFAR-10} The optimal parameters for random and class forgetting respectively $\eta = (1e-2, 1e-2)$, retain batch size = $(128, 128)$.
Training was performed for 100 epochs for ImageNet-based models and 30 epochs for CIFAR-10.

\paragraph{\ourmethod{}.} We performed grid search for the combination ($\alpha$) and learning rate ($\eta$), and batch sizes for $\mathcal{D}_u$ and $\mathcal{D}_r$. We apply LoRA modules to all linear layers within the self-attention and cross-attention layers. We set the rank to $8$ and the scaling factor to $32$. Specifically, for $\alpha$, we searched over $\{0.9, 0.8\}$, and for $\eta$, we searched over $\{0.001, 0.01\}$. For the retain batch size, we considered $\{256, 128\}$, and for the unlearn batch size, we searched over $\{256, 128\}$.
\textbf{For ViT architecture} The optimal parameters for random and class forgetting respectively $\alpha = (0.9, 0.8)$,  $\eta = (0.001, 0.001)$, retain batch size = $(128, 256)$, unlearn batch size = $(128, 256)$. \textbf{For ResNet18 architecture} The optimal parameters for random and class forgetting respectively $\alpha = (0.9, 0.8)$,  $\eta = (0.001, 0.001)$, retain batch size = $(256, 256)$, unlearn batch size = $(128, 256)$.
 
\paragraph{NegGrad.} We performed grid search for learning rate ($\eta$), and batch size for $\mathcal{D}_u$. Specifically, for $\eta$, we searched over $\{1e-3, 1e-4, 1e-5, 1e-6\}$, and for unlearn batch size, we searched over $\{256, 128\}$.
\textbf{For ViT architecture} The optimal parameters for random and class forgetting respectively $\eta = (1e-4, 1e-4)$, unlearn batch size = $(256, 128)$. \textbf{For ResNet18 architecture} The optimal parameters for random and class forgetting respectively $\eta = (1e-5, 1e-5)$, unlearn batch size = $(128, 256)$.

\paragraph{NegGrad+.} We performed grid search for learning rate ($\eta$), and batch sizes for $\mathcal{D}_u$ and $\mathcal{D}_r$. Specifically, for $\eta$, we searched over $\{1e-3, 1e-4, 1e-5, 1e-6\}$. For the retain batch size, we considered $\{256, 128\}$, and for unlearn batch size, we searched over $\{256, 128\}$.
\textbf{For ViT architecture} The optimal parameters for random and class forgetting respectively $\eta = (1e-3, 1e-3)$, retain batch size = $(128, 256)$, unlearn batch size = $(128, 256)$. \textbf{For ResNet18 architecture} The optimal parameters for random and class forgetting respectively $\eta = (1e-6, 1e-5)$, retain batch size = $(256, 256)$, unlearn batch size = $(256, 256)$.

\paragraph{Fisher.} We performed grid search for  ($\alpha$), and batch size for $\mathcal{D}_r$. Specifically, for $\alpha$, we searched over $\{1e-7, 1e-8, 1e-9\}$, and for retain batch size, we searched over $\{128, 256\}$.
\textbf{For ViT architecture} The optimal parameters for random and class forgetting respectively $\eta = (1e-9, 1e-9)$, retain batch size = $(128, 256)$. \textbf{For ResNet18 architecture} The optimal parameters for random and class forgetting respectively $\eta = (1e-9, 1e-9)$, retain batch size = $(128, 256)$.

\paragraph{Influence.} We performed grid search for ($\alpha$), and batch sizes for $\mathcal{D}_u$ and $\mathcal{D}_r$. Specifically, for $\alpha$, we searched over $\{1, 0.1, 0.01, 1e-3\}$. For the retain batch size, we considered $\{64, 128, 256\}$, and for unlearn batch size, we searched over $\{64, 128, 256\}$.
\textbf{For ViT architecture} The optimal parameters for random and class forgetting respectively $\eta = (1, 1)$, retain batch size = $(256, 128)$ and for unlearn batch size = $(128, 128)$. \textbf{For ResNet18 architecture} The optimal parameters for random and class forgetting respectively $\eta = (0.001, 1)$, retain batch size = $(64, 256)$ and for unlearn batch size = $(256, 128)$.

\paragraph{SCRUB.} We performed grid search for learning rate ($\eta$), and batch sizes for $\mathcal{D}_u$ and $\mathcal{D}_r$. Specifically, for $\eta$, we searched over $\{5e-2, 5e-5-3, 5e-4\}$. For the retain batch size, we considered $\{256, 512\}$, and for unlearn batch size, we searched over $\{256, 512\}$. \textbf{For ViT architecture} $\eta = 5e-3$ retain batch size = 256, unlearn batch size = 512. \textbf{For ResNet18 architecture} $\eta = 5e-4$ retain batch size = 256, unlearn batch size = 512.

\begin{figure*}[t]
    \centering
    \begin{minipage}[t]{0.49\textwidth}
        \centering
        \includegraphics[width=\textwidth]{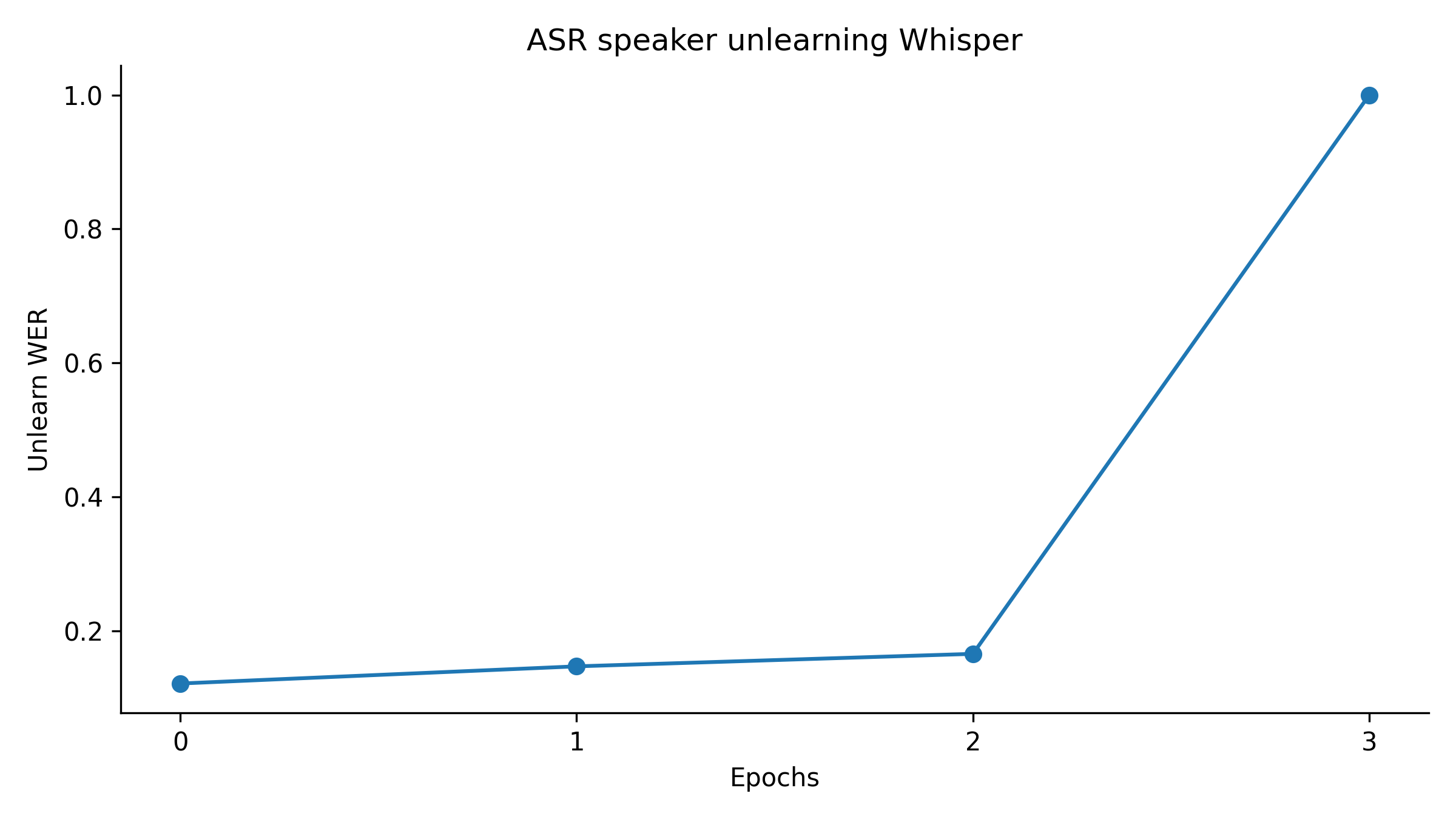}
        \caption{\textit{Unlearn WER across epochs for ASR speaker unlearning with Whisper using OrthoGrad.} The metric remains low in early epochs and exhibits a sharp increase at later epochs, highlighting the phase-transition behavior observed across methods.}
        \label{fig:wer_jump_placeholder}
    \end{minipage}\hfill
    \begin{minipage}[t]{0.49\textwidth}
        \centering
        \includegraphics[width=\textwidth]{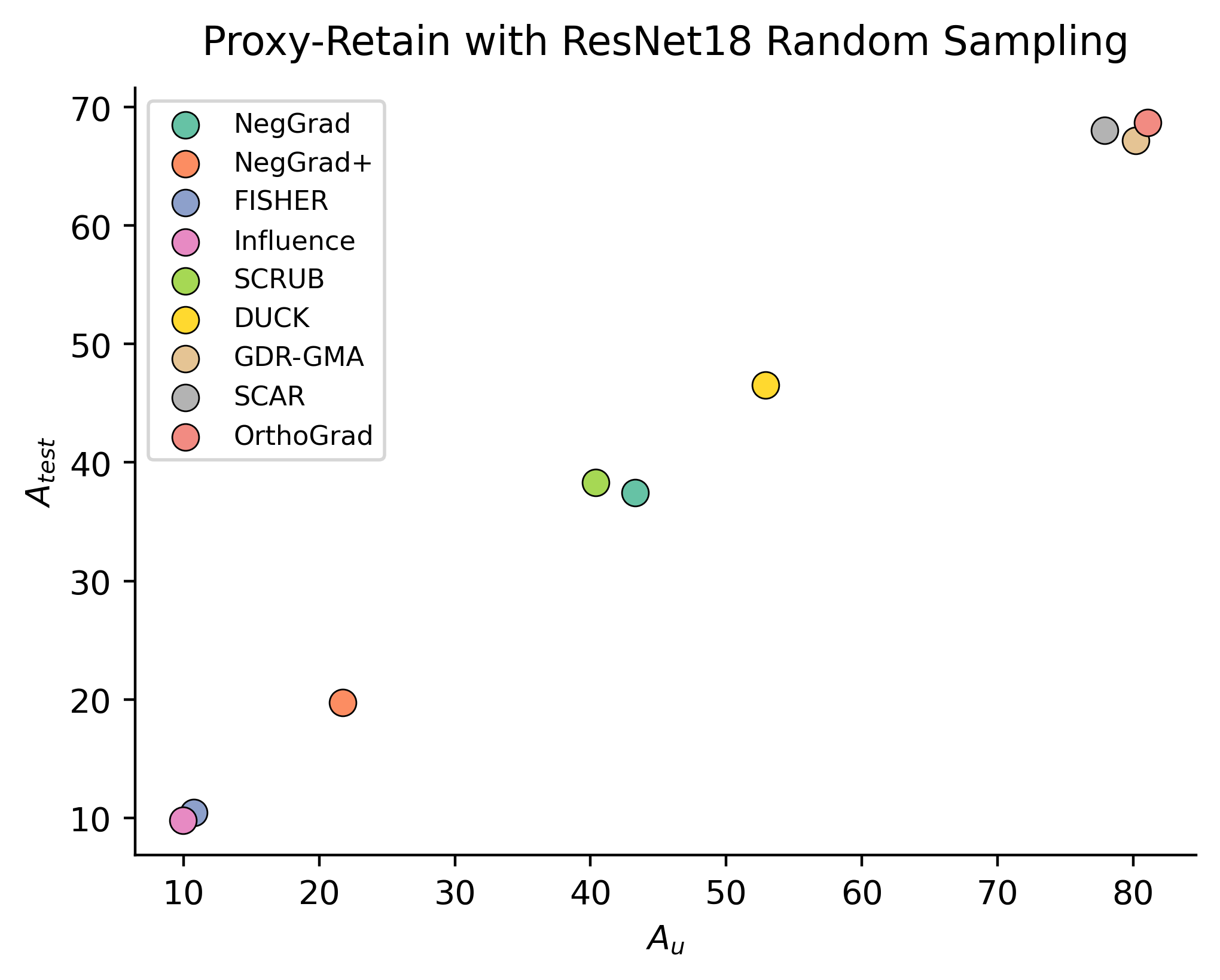}
        \caption{\textit{Proxy-retain evaluation under random sampling with a ResNet-18 model.} Illustrating the trade-off between forgetting (lower $A_u$) and retained generalization (higher $A_{\text{test}}$).}
        \label{fig:test_unlearn_tradeoff}
    \end{minipage}
\end{figure*}

\paragraph{DUCK.} We performed grid search for learning rate ($\eta$), and batch sizes for $\mathcal{D}_u$ and $\mathcal{D}_r$. Specifically, for $\eta$, we searched over $\{2e-4, 5e-4, 5e-5\}$. For the retain batch size, we considered $\{128, 1024\}$, and for unlearn batch size, we searched over $\{128, 1024\}$. \textbf{For ViT architecture} $\lambda_{fgt} = (1.5, 0.5)$, $\lambda_{ret} = (1.5, 1.5)$, batch ratio = $(5, 30)$, $\eta = (5e-5, 2e-4)$, retain and unlearn batch size = $(128, 128)$, $\tau$ = $(3,3)$. \textbf{For ResNet18 architecture} $\lambda_{fgt} = (1.5, 0.5)$, $\lambda_{ret} = (1.5, 1.5)$, batch ratio = $(5, 30)$, $\eta = (5e-4, 2e-4)$, retain and unlearn batch size = $(1024, 1024)$, temperature = $(3,3)$.

\paragraph{SCAR.} We performed grid search for learning rate ($\eta$). Specifically, for \(\eta \),  we searched over \( \{ 1e-4, 5e-4, 1e-3\}\). In addition, to ensure a fair comparison and maintain consistency with prior work, we adopted the remaining hyperparameters as reported in~\cite{bonato2024retain} (refer to Table~9 in~\cite{bonato2024retain} for details).

\paragraph{SSD.} We performed grid search for ($\alpha$), and batch sizes for $\mathcal{D}_u$ and $\mathcal{D}_r$. Specifically, for $\alpha$, we searched over $\{1.1, 1.3, 1.5, 1.7, 5, 10, 30, 50\}$. For the retain batch size, we considered $\{64, 128, 256\}$, and for unlearn batch size, we searched over $\{64, 128, 256\}$.
\textbf{For ViT architecture} The optimal parameters for random and class forgetting respectively $\eta = (1.3, 30)$, retain batch size = $(128, 128)$ and for unlearn batch size = $(128, 256)$. \textbf{For ResNet18 architecture} The optimal parameters for random and class forgetting respectively $\eta = (10, 10)$, retain batch size = $(128, 128)$ and for unlearn batch size = $(256, 256)$.

\paragraph{GDR-GMA.} We performed grid search for learning rate ($\eta$), and batch sizes for $\mathcal{D}_u$ and $\mathcal{D}_r$. Specifically, for $\eta$, we searched over $\{1e-3, 1e-4, 1e-5\}$. For the retain batch size, we considered $\{128, 256\}$, and for unlearn batch size, we searched over $\{128, 256\}$. For all setups and architectures $\eta = 1e-4$ retain batch size = $256$, unlearn batch size = $256$.
\subsection{Automatic Speech Recognition}
\label{app:asr}
Here, we present additional details about the ASR machine unlearning setup. This includes both general information about the setup and the hyperparameter search conducted for each method. In this setup, we utilized the train-$100$ split from the LibriSpeech dataset, targeting unlearning samples from a single speaker. The results are reported across $5$ randomly sampled speakers. For all methods, we utilize a batch size of $48$ for the retain and unlearn sets and $30$ epochs with early stopping. In addition, we use the well-known Adam~\cite{kingma2014adam} as the optimizer.

\paragraph{\ourmethod{}.} - We performed grid search for the combination ($\alpha$) and learning rate ($\eta$) parameters. Specifically for $\alpha$ we searched over $\{0.05, 0.2, 0.35, 0.5\}$ and $\{1e-5, 5e-6, 1e-6\}$ for $\eta$. We apply LoRA modules to all linear layers within the self-attention and cross-attention layers. We set the rank to $8$ and the scaling factor to $32$.

\paragraph{Finetune.} We performed a grid search for the learning rate parameter. Specifically, we explored learning rates in the range $\{1e-5, 5e-6, 1e-6\}$, and the optimal learning rate chosen is $\eta = 1e-5$.

\paragraph{NegGrad+.} - We performed a grid search for the learning rate parameter. Specifically we searched over $\{1e-5, 5e-6, 2.5e-5, 1e-6, 5e-7, 1e-7\}$, and the optimal learning rate chosen is $5e-6$. 

\paragraph{SCRUB.} We performed a grid search for the learning rate and the number of epochs in which SCRUB performs unlearning steps. We searched over the range $[1e-4, 1e-7]$ with a step size of $0.5$ in multiplication, as well as $\{10, 20, 30\}$ for the number of unlearning epochs. We observed that SCRUB either failed to achieve unlearning entirely or caused the model to collapse, resulting in poor generalization. We report the results with $1e-5$ learning and $20$ unlearning epochs. We also used temperature rescaling of $4$ in the knowledge distillation loss and $5e-4$ weight decay.

\paragraph{GDR-GMA.} We performed a grid search for the learning rate and learning rate scheduling factor parameters. We searched over the range $[1e-4, 1e-7]$ with a step size of $0.5$ in multiplication, and found the optimal learning rate to be $1e-4$. We explored learning rate scheduling factors in the set $\{10, 50, 100, 200\}$ and identified $100$ as the optimal value. 

\subsection{Data preprocessing}
We adopt the data preprocessing approach outlined by~\cite{radford2023robust}. The audio samples are resampled to $16$ kHz, and log-magnitude Mel-spectrograms are generated. Specifically, we compute $80$-channel Mel-spectrograms using $25$-millisecond windows with a $10$-millisecond stride.

\section{Limitations}\label{sec:limit}
This work focuses on the low-data regime, showing strong gains over baselines in this setting. However, per-sample gradient handling increases GPU memory use. Despite these limitations, \ourmethod{} presents a compelling solution in data-constrained environments.

\end{document}